\documentclass[11pt]{article}
\usepackage{fullpage}
\usepackage{hyperref}       
\usepackage{url}    
\usepackage{booktabs}       
\usepackage{nicefrac}       
\usepackage{microtype}  
\usepackage{times}   
\usepackage{xcolor,colortbl} 
\usepackage{multirow}
\usepackage{tablefootnote}
\usepackage{graphicx}
\usepackage{subfigure}

\usepackage{preamble}
\usepackage{mathnotations}
\usepackage{cleveref}

\usepackage[style=alphabetic,maxnames=12,maxbibnames=10,maxcitenames=10,maxalphanames=10,giveninits=true,doi=false,url=true]{biblatex}
\newcommand*{\citet}[1]{\AtNextCite{\AtEachCitekey{\defcounter{maxnames}{2}}} \textcite{#1}}

\newcommand*{\citep}[1]{\cite{#1}}

\begin{document}
\title{On Private and Robust Bandits}
\author {
  Yulian Wu \thanks{Equal contribution.}  \thanks{King Abdullah University of Science and Technology, Thuwal, Saudi Arabia. Email: \texttt{yulian.wu@kaust.edu.sa}} \quad
     Xingyu Zhou \footnotemark[1]  \thanks{Wayne State University, Detroit, USA.  Email: \texttt{xingyu.zhou@wayne.edu}}\quad   
     Youming Tao \thanks{Shandong University, Qingdao, China. Email: \texttt{ym.tao99@mail.sdu.edu.cn }} \quad
     Di Wang \thanks{King Abdullah University of Science and Technology, Thuwal, Saudi Arabia. Email: \texttt{di.wang@kaust.edu.sa}} 
}

\date{}

\maketitle

\begin{abstract}
We study private and robust multi-armed bandits (MABs), where the agent receives Huber's contaminated heavy-tailed rewards and meanwhile needs to ensure differential privacy. We first present its minimax lower bound, characterizing the information-theoretic limit of regret with respect to privacy budget, contamination level and heavy-tailedness. Then, we propose a meta-algorithm that builds on a private and robust mean estimation sub-routine \texttt{PRM} that essentially relies on reward truncation and the Laplace mechanism only.  For two different heavy-tailed settings, we give specific schemes of \texttt{PRM}, which enable us to achieve nearly-optimal regret. As by-products of our main results, we also give the first minimax lower bound for private heavy-tailed MABs (i.e., without contamination). Moreover, our two proposed truncation-based \texttt{PRM} achieve the optimal trade-off between estimation accuracy, privacy and robustness. Finally, we support our theoretical results with experimental studies.
\end{abstract}

\newpage
\tableofcontents
\newpage

\section{Introduction}
The multi-armed bandit (MAB)~\cite{berry1985bandit} problem provides a fundamental framework for sequential decision-making under uncertainty with bandit feedback, which has drawn a wide range of applications  in medicine \cite{gutierrez2017multi}, finance \cite{hoffman2011portfolio,shen2015portfolio}, recommendation system \cite{caron2013mixing}, and online advertising \cite{schwartz2017customer}, to name a few. Consider a 
portfolio selection  in finance as an example. 
At each decision round $t \in [T]$, the learning agent selects an action $a_t \in [K]$ (i.e., a particular choice of assets to user $t$) and receives a reward $r_{t}$ (e.g., the corresponding payoff) that is i.i.d. drawn  from an unknown probability distribution associated with the portfolio choice.  The goal is to learn to maximize its cumulative payoff.

In practice, applying the celebrated MAB formulation to real-life applications (e.g., the above finance example) needs to deal with both robustness and privacy issues. On the one hand, it is known that finical data is often heavy-tailed (rather than sub-Gaussian)~\cite{rachev2003handbook,hull2012risk}. Moreover, the received payoff data in finance often contains outliers~\cite{adams2019identifying} due to data contamination. On the other hand, privacy concern in finance is growing \cite{lei2020privacy,chen2022privacy,chen2022differential}. For instance, even if the adversary does not have direct access to the dataset, they
are still able to reconstruct other customers’ personal information by interacting with the pricing platform and observing  its decisions~\cite{fredrikson2014privacy}.

Motivated by this, a line of work on MABs has focused on designing robust algorithms with respect to heavy-tailed rewards~\cite{bubeck2013bandits}, adversary contamination \cite{lykouris2018stochastic,kapoor2019corruption}, or both \cite{basu2022bandits}. Another line of recent work has studied privacy protection in MABs via different trust models of differential privacy (DP) such as central DP~\cite{sajed2019optimal,azize2022privacy}, local DP~\cite{ren2020multi,tao2021optimal} and distributed DP~\cite{tenenbaum2021differentially,chowdhury2022distributed}. Moreover,  there have also been recent advances in understanding  the close relationship between robustness and privacy for the mean estimation problem (e.g., robustness induces privacy \cite{hopkins2022robustness} and vice versa \cite{georgiev2022privacy}). In light of this, a fundamental question we are interested in this paper is:
\begin{center}
   \emph{Is there a simple algorithm that can tackle privacy and robustness in MABs simultaneously?} 
\end{center}
\textbf{Our contributions.} We give an affirmative answer to it  by showing that a simple truncation-based algorithm could achieve a nearly optimal trade-off between regret, privacy, and robustness for MABs. The key intuition is that reward truncation not only helps to reduce outliers (due to both heavy tails and contamination), but bound also its sensitivity, which is necessary for DP. To make our intuition rigorous, we take the following principled approaches.

\textbf{(i)} We first establish the minimax regret lower bound for private and robust MABs, i.e., heavy-tailed MABs with both privacy constraints and Huber's contamination~\cite{huber1964robust} (see~\cref{sec:lower}). This characterizes the information-theoretic limit of regret with respect to privacy budget, contamination level and  heavy-tailedness. As a byproduct, our result also implies  the first minimax lower bound for private heavy-tailed MABs (i.e., without contamination), hence resolving an open problem in~\citet{tao2021optimal}.

\textbf{(ii)} To match the lower bound, we first propose a meta-algorithm (see~\cref{sec:alg}), which builds upon the idea of batched successive elimination and relies on a generic private and robust mean estimation sub-routine denoted by \texttt{PRM}. Then, for two different settings of (heavy-tailed) reward distributions (i.e., finite raw or central moments), we propose corresponding schemes for the sub-routine \texttt{PRM}, both of which only require truncation and the Laplace mechanism to guarantee robustness and privacy, simultaneously. Armed with these, our meta-algorithm can  enjoy nearly matching regret upper bounds (see~\cref{sec:upper}). Experimental studies also corroborate our theoretical results. 

\textbf{(iii)} Along the way, several results could be of independent interest. In particular, our proposed \texttt{PRM} shows that truncation is sufficient to help achieve the optimal high-probability concentration for private and robust mean estimation in the one-dimension case. Moreover, without contamination, our regret upper bounds not only match the optimal one for private heavy-tailed MABs with finite raw moments, but also provide the first results for the case with finite central moments, hence a complete study for private bandits.

Due to space limit, technical lemmas and all proofs  are included in Appendix. 

\section{Related Work}

\textbf{Robust MABs.} The studies on robust bandits can be largely categorized into two groups. The first group of work mainly focuses on the setting where the total contamination is bounded, i.e., the cumulative difference between observed reward and true reward is bounded by some constant~\cite{lykouris2018stochastic}. The second group considers Huber's $\alpha$-contamination model~\cite{huber1964robust} (which is also the focus of our paper) or a similar $\alpha$-fraction model. In these cases,  the reward for each round can be contaminated by an arbitrary distribution with probability $\alpha \in [0,1]$~\cite{awasthi2020online,kapoor2019corruption,mukherjee2021mean}, or at most $\alpha$-fraction of the rewards are arbitrarily contaminated~\cite{niss2020you}. The existing work in this group has mainly focused on  the light-tailed setting where the true inlier distribution is Gaussian or sub-Gaussian and uses a robust median or trimmed-mean estimator. A very recent work~\cite{basu2022bandits} studies the setting where the inlier distribution only has finite variance and uses Huber's estimator to establish problem-dependent bounds. In contrast, we take the perspective of minimax regret, i.e., problem-independent bounds, and also account for privacy.


\noindent \textbf{Private MABs.} In addition to the above mentioned results on private MABs with light-tailed rewards,~\citet{tao2021optimal} study private heavy-tailed MABs with finite raw moments under both central and local models of DP.  However, the optimal minimax regret for this case is still unknown and how to design private algorithms for heavy-tailed distributions with finite \emph{central} moments is unclear. In this paper, as byproducts of our main results, we resolve both problems. 


\noindent\textbf{Robust and private mean estimation.}  
Our work is also related to robust and private mean estimation, especially the one-dimensional case. On the robustness side with Huber's model, a high-probability concentration bound for the median of Gaussian (hence the mean by symmetry) is first established in~\citet{lai2016agnostic}. Recently,~\citet{mukherjee2021mean} give a high probability mean concentration via a trimmed-mean estimator for general sub-Gaussian inlier distributions while~\citet{prasad2019unified} focus on the heavy-tailed setting. On the privacy side, one close work is~\citet{kamath2020private}, which presents the first high-probability mean concentration for private heavy-tailed distributions with finite central moments (via a medians-of-means approach). It is worth noting  that there are recent exciting advances in understanding the close relationship between robustness and privacy (e.g., robustness induces privacy \cite{hopkins2022robustness} and vice versa \cite{georgiev2022privacy}). From this aspect, our results imply that for the one-dimensional mean estimation problem, truncation alone suffices to help to achieve both.

\vspace{-0.1in}
\section{Preliminary}
In this section, we first formally introduce our private and robust MAB problem and then present its regret notions.
\subsection{Private and Robust MABs}

As mentioned before, by robustness, we aim to handle both reward contamination and possible heavy-tailed inlier distributions. To this end, we first introduce the following two classes of heavy-tailed reward distributions.  




\begin{definition}[Finite $k$-th raw moment]
\label{def:raw}
A distribution over $\mathbb{R}$ is said to have a finite $k$-th raw moment if it is within
\begin{align}\label{eq:AssumpHeavy}
\mathcal{P}_k=\left\{P: \mathbb{E}_{X\sim P}\left[|X|^k\right] \leq 1\right\},\quad k\ge 2.
\end{align}
\end{definition}

\begin{definition}[Finite $k$-th central moment]
\label{def:central}
A distribution over $\mathbb{R}$ is said to have a finite $k$-th central moment if it is within
\begin{align}\label{eq:AssumpCentralHeavy}
{\mathcal{P}_k^c}=\left\{P:\mathbb{E}_{X \sim P}\left[|X-\mu|^k\right] \leq 1\right\}, \quad k\ge 2.
\end{align}
where $\mu:= \mathbb{E}_{ X \sim P}[X] \in [-D,D]$, i.e., $D$ is the finite range of its mean but can be arbitrarily large. 
\end{definition}

 We further consider the celebrated Huber contamination model~\cite{huber1964robust} and apply it to heavy-tailed MABs. 

\begin{definition}[Heavy-tailed MABs with Huber contamination]
\label{eq:AssumpHuber} Given the corruption level $\alpha \in [0,1]$. 
For each round $t \in [T]$, the observed reward\footnote{Here we use $x_t$ in the contaminated case to distinguish with standard reward $r_t$. } $x_t$ for action $a_t$, is sampled independently from the true distribution $P_{a_t} \in \cP_k $ (or $P_{a_t} \in \cP_k^c$) with probability $1-\alpha$; otherwise is sampled from some arbitrary and unknown contamination distribution $G_{a_t} \in \cG$. 

\end{definition}

In addition to robustness, we also consider the privacy protection in MABs via the lens of DP. 
In particular, we consider the standard central model of DP for MABs (e.g.,~\cite{mishra2015nearly}), where the learning agent has access to users' raw data (i.e., rewards) and guarantees that its output (i.e., sequence of actions) are indistinguishable in probability on two neighboring reward sequences. Due to contamination, the reward data accessed by the learning
agent at round $t$ could have already been contaminated. More precisely, we let $D_T=(x_{1},\ldots,x_{T}) \in \mathbb{R}^T$ be a reward sequence generated in the learning process and $\cM(D_T)=(a_1,\ldots,a_T) \in [K]^T$ to denote the sequence of all actions recommended by a learning algorithm $\cM$. With this setup, we have the following formal definition.

\begin{definition}[Differential Privacy for MABs]
For any $\epsilon > 0$, a learning algorithm $\cM: \mathbb{R}^T \rightarrow [K]^{T}$ is $\epsilon$-DP if for all sequences $D_T,D_T^\prime \in \mathbb{R}^T$ differing only in a single element and for all events $E \subset [K]^{T}$, we have
\begin{align*}
   \mathbb{P}\left[\mathcal{M}{\left(D_{T}\right) \in E}\right] \leq e^{\epsilon} \cdot \mathbb{P}\left[\mathcal{M}\left(D_{T}^{\prime}\right) \in E\right].
\end{align*}
\end{definition}

In this paper, we will leverage the well-known Laplace mechanism to guarantee differential privacy. 

\begin{definition}[Laplace Mechanism]\label{def:3}
Given a function $f : \mathcal{X}^n\rightarrow \mathbb{R}^d$, the Laplacian mechanism is given by 
\begin{align*}
  \mathcal{M}_L(D,f,\epsilon)=f(D)+ (Y_1, Y_2, \cdots, Y_d),  
\end{align*}
where $Y_i$ is i.i.d. drawn from a Laplacian Distribution\footnote{For a parameter $\lambda$, the Laplacian distribution has the density function $\text{Lap}(\lambda) (x)=\frac{1}{2\lambda}\exp(-\frac{|x|}{\lambda})$.} $\text{Lap}(\frac{\Delta_1(f)}{\epsilon})$,  where $\Delta_1(f)$ is the $\ell_1$-sensitivity of the function $f$, {\em i.e.,}
$\Delta_1(f)=\sup_{D\sim D'}||f(D)-f(D')||_1.$ Then, for any $\epsilon >0$, 
Laplacian mechanism satisfies $\epsilon$-DP.
\end{definition}

In the following sections, for brevity, we will simply use \emph{private and robust MABs} to refer to our setting, i.e., heavy-tailed MABs with Huber contamination and privacy constraints. 


\subsection{Regrets for Private and Robust MABs}
In the contamination case, the standard regret using observed (contaminated) rewards $\{x_t\}_{t \in [T]}$ is ill-defined~\cite{niss2020you}. Instead, the literature focuses on the \emph{clean regret}, that is, to compete with the best policy in hindsight as measured by the expected true uncontaminated rewards~\cite{niss2020you,basu2022bandits,chen2022online}. Hence, let $\mu_a$ be the mean of the inlier distribution of arm $a \in [K]$ and $\mu^* = \max_{a\in [K]}\mu_a$. We also let $\Pi^{\epsilon}$ be the set of all $\epsilon$-DP MAB algorithms and $\cE_{\alpha,k}$ be the set of all instances of heavy-tailed MABs with Huber contamination.



\begin{definition}[Clean Regret]
Fix an algorithm $\pi \in \Pi^{\epsilon}$ and an instance $\nu \in \cE_{\alpha,k}$. Then, the clean regret of $\pi$ under $\nu$ is given by 
\begin{equation*}
    \mathcal{R}_T(\pi,\nu):=\mathbb{E}_{\pi,\nu} [T\mu^* - \sum_{t=1}^T\mu_{a_t}].
\end{equation*}
\end{definition}
Note that here the expectation is taken over the randomness generated by the \emph{contaminated} environment and $\epsilon$-DP MAB algorithm while the means are of the true inlier distributions. 

To capture the intrinsic difficulty of the private and robust MAB problem, we are also interested in its minimax regret. 
\begin{definition}[Minimax Regret]
The minimax regret of our private and robust MAB problem is defined as 
\begin{equation}\label{eq:minimaxReg}
    \mathcal{R}^{\text{minimax}}_{\epsilon,\alpha,k}:= \inf_{\pi \in \Pi^{\epsilon}}\sup_{\nu \in \cE_{\alpha, k}} \mathbb{E}_{\pi, \nu}[T\mu^* - \sum_{t=1}^T\mu_{a_t}].
\end{equation}
\end{definition}

\section{Lower Bound}
\label{sec:lower}
We start with the following lower bound on the minimax regret, which characterizes the fundamental impact of privacy budget (via $\epsilon$), contamination level (via $\alpha$) and heavy-tailedness of rewards (via $k$) in the regret.  

\begin{theorem}
\label{thm: lowerBound}
    Consider a private and robust MAB problem where inlier distributions have finite $k$-th raw (or central) moments ($k\ge 2$). Then, its minimax regret satisifes
    \begin{align*} \mathcal{R}^{\text{minimax}}_{\epsilon,\alpha, k}&=\Omega\left(\sqrt{KT}+\left(\frac{K}{\epsilon}\right)^{1-\frac{1}{k}}T^{\frac{1}{k}}+ T \alpha^{1-\frac{1}{k}}\right).
    \end{align*}
\end{theorem}

Let us first present interpretations of the above result, which basically takes a maximum of three terms. The first term comes from the standard regret for Gaussian rewards, the second one captures the additional cost in regret due to privacy and heavy-tailed rewards, and the last term indicates the additional cost in regret due to contamination and heavy-tailed rewards. Note that, for a given $k$, the impact of privacy and contamination is separable. It would also be useful to compare our lower bound with the related ones, which is the purpose of the following remark. 
\begin{remark}
First, when $k = \infty$ and $\alpha =0$, our lower bound recovers the state-of-the-art lower bound for private MABs with sub-Gaussian rewards~\cite{azize2022privacy}; Second, we note that even when $\alpha=0$, there is no existing result on minimax regret (i.e., problem-independent) lower bound for private heavy-tailed MABs. In fact, this is left as an open problem in a recent work~\cite{tao2021optimal}. Thus, our lower bound not only resolves the problem\footnote{In~\cite{tao2021optimal}, the authors consider a slightly different setting where the heavy-tailed distribution only has a finite $(1+v)$-th moment with $v \in (0,1]$. However, our result simply generalizes to this setting by taking $k = 1+v$.}, but also captures contamination as well. Finally, when there is no privacy protection, a very recent work~\cite{basu2022bandits} establishes a \emph{problem-dependent}  regret lower bound for robust MABs while we are interested in problem-independent lower bound. Thus, its results is incomparable to ours. 
\end{remark}
Now, it remains to see whether this lower bound can be achieved via certain algorithms, which is the main focus of the following two sections.

\section{Our Approach: A Meta-Algorithm}
\label{sec:alg}
In this section, we first introduce a meta-algorithm for private and robust MABs, which not only allows us to tackle inlier distributions with bounded raw or central moments in a unified way, but also highlights the key component, i.e., a private and robust mean estimation sub-routine building on the simple idea of truncation.


Our meta-algorithm, at a high level, can be viewed as a batched version of the celebrated successive arm elimination~\cite{even2006action} along with a private and robust mean estimation sub-routine $\texttt{PRM}$ (see Algorithm~\ref{Alg:meta}). That is, it divides the time horizon $T$ into batches with exponentially increasing
size and eliminates sub-optimal arms successively based on the mean estimate via \texttt{PRM}. More specifically, based on the batch size, it consists of two phases. That is, when the batch size is less than a threshold $\cT$, it simply recommends actions randomly (\textcolor{blue}{line 5-7}) (more on this will be explained soon). Otherwise, for each active arm $a$ in batch $\tau$, it first prescribes $a$ to a batch of $B_{\tau} = 2^\tau$ fresh \emph{new} users and observes possibly contaminated rewards (\textcolor{blue}{line 8}). Then, it calls the sub-routine \texttt{PRM} to compute a private and robust mean estimate for each active arm $a$ (\textcolor{blue}{line 12}). In particular, it \emph{only} uses the rewards within the most recent batch (i.e., ``forgetting'') along  with a proper reward truncation threshold $M_{\tau}$.  Finally, it adopts the classic idea of arm elimination with a proper choice of confidence radius $\beta_{\tau}$ to remove sub-optimal arms with high confidence (\textcolor{blue}{line 18-20}).

\begin{algorithm}[t!]
    \caption{Private and Robust Arm Elimination}
    \label{Alg:meta}
    \begin{algorithmic}[1]
        \STATE {\bfseries Input:} Number of arms $K$, time horizon $T$, privacy budget $\epsilon$, Huber parameter $\alpha \in (0,1]$, error probability $\delta \in (0,1]$, inliner distribution parameters i.e., $k$ and optional $D$
        \STATE Initialize:  $\tau=0$, active set of arms $\mathcal{S}=\{1,\cdots,K\}$.
        \FOR{ batch $\tau=1,2,\dots$}
            \STATE Set batch size $B_\tau= 2^\tau$
            \IF{$B_\tau < \cT$}
                \STATE Randomly select an action $a \in [K]$
                \STATE Play action $a$ for $B_{\tau}$ times
            \ELSE 
                
            \FOR {each active arm $a\in\mathcal{S}$}
            \FOR {$i$ from $1$ to $B_\tau$}
               \STATE Pull arm $a$, observe contaminated reward $x_{i}^{a}$ 
               \STATE If total number of pulls reaches $T$, \textbf{exit}
            \ENDFOR  
            \STATE Set truncation threshold $M_\tau$
            \STATE Set additional parameters $\Phi$
             \STATE Compute  estimate $\widetilde{\mu}_a = \texttt{PRM}(\{x_i^{a}\}_{i=1}^{B_{\tau}}, M_{\tau}, \Phi )$
            \ENDFOR
            \STATE Set confidence radius $\beta_\tau$
             \STATE Let $\widetilde{\mu}_{\rm max}=\max_{a\in\mathcal{S}}\widetilde{\mu}_a$
             \STATE Remove all arms $a$ from $\mathcal{S}$ s.t. $\widetilde{\mu}_{\rm max}-\widetilde{\mu}_a> 2\beta_\tau$
            \ENDIF
            
        \ENDFOR
    \end{algorithmic}
\end{algorithm}

We now provide more intuitions behind our algorithm design by highlighting how its main components work in concert. \emph{First,} the reason behind the first phase (i.e., $B_{\tau} \le \cT$) is that the mean estimate by \texttt{PRM} does not have a high probability concentration when the sample size is small. Thus, one cannot adopt arm elimination in this phase since it might eliminate the optimal arm. Note that, instead of our choice of random selection, one can also use other methods for the first phase (see Remark~\ref{rem:first-phase} below). \emph{Second,} for the second phase, the idea of batching and forgetting is the key to achieving privacy with a minimal amount of noise (hence better regret). This is because now any single reward feedback only impacts one computation of estimate. This is in sharp contrast to standard arm elimination (e.g.,~\cite{even2006action}) where each mean estimate is based on all samples so far (as no batching is used), and hence a single reward change could impact $O(T)$ mean estimations\footnote{One can use tree-based algorithm~\cite{chan2011private} to reduce it to $O(\log T)$, but it is still sub-optimal~\cite{sajed2019optimal}.}. \emph{Third,} the simple idea of reward truncation in \texttt{PRM} turns out to be extremely useful for both robustness and privacy. On the one hand, truncation helps to reduce the impact of outliers (due to both heavy tails and contamination); On the other hand, truncation also helps to bound the sensitivity, which is necessary for privacy. In fact, as we will show later, a well-tuned truncation threshold enables us to achieve a near-optimal trade-off between regret, privacy and robustness. \emph{Finally,} in contrast to the first phase, we can now eliminate sub-optimal arms with high confidence due to the high probability concentration of mean estimate when batch size is larger than $\cT$ (more details will be given later for specific choices of \texttt{PRM} and hence the choice of $\cT$).

\begin{remark}
\label{rem:first-phase}
    The algorithm choice of the first phase can be flexible. For example, instead of playing a randomly selected action for the whole batch, one can choose to play a randomly selected action for each round. Moreover, one can also choose to be greedy or probabilistically greedy with respect to the mean estimate by \texttt{PRM}, which also only uses the rewards collected within the last batch for each arm. All of these choices have the same theoretical guarantees, though some will help to improve the empirical performance.
\end{remark}

We then present the following remark that places our meta-algorithm in the existing literature. 
\begin{remark}[Comparison with existing literature]
For private MABs (without contamination),  the state-of-the-art also builds upon the idea of batching and forgetting~\cite{sajed2019optimal,chowdhury2022distributed} to achieve optimal regret. For robust MABs (without privacy), existing works take different robust mean estimations. For example, both~\citet{niss2020you,mukherjee2021mean} use a trimmed mean estimator for sub-Gaussian inlier distributions while~\citet{basu2022bandits} adopts Huber's estimator to handle inlier distributions with only bounded variance. We are the first to study privacy and robustness simultaneously, via a simple truncation-based estimator, which in turn reveals the close relationship between privacy and robustness in MABs. This complements the recent advances in capturing the connection between these two in (high-dimensional) statistics~\cite{hopkins2022robustness,georgiev2022privacy}. 
\end{remark}

\section{Upper Bounds}
\label{sec:upper}
In this section, we establish the regret upper bounds for two specific instantiations of our meta-algorithm, i.e., one for the finite raw moment case and another for the finite central moment case. In particular, the results could match our lower bound up to a logarithmic factor, demonstrating their near-optimality. 

\begin{algorithm}[t!]
    \caption{\texttt{PRM} for the finite raw moment case}
    \label{Alg:prm-raw}
    \begin{algorithmic}[1]
        \STATE {\bfseries Input:} A collection of data $\{x_i\}_{i=1}^n$, truncation parameter $M$, additional parameters $\Phi = \{\epsilon\}$ 
        \FOR{ $i=1,2,\dots, n$}
            \STATE Truncate data $\bar{x}_i= x_i \cdot\mathbbm{1}_{\{|x_{i}|\le M\}}$
        \ENDFOR
        \STATE Return private estimate $\widetilde{\mu} = \frac {\sum_{i=1}^{n}\bar{x}_{i}}{n} +{\rm Lap}(\frac{2M}{n\epsilon})$
    \end{algorithmic}
\end{algorithm}

\subsection{Finite Raw Moment Case}
\label{sec:rawupper}
In this section, we will focus on private and robust MABs where the inlier distributions have a finite $k$-th raw moment as given by Definition~\ref{def:raw}. 
In particular, we first introduce the choice of \texttt{PRM} in this case (see Algorithm~\ref{Alg:prm-raw}) and establish its concentration property,  which plays a key role in our implementation of meta-algorithm. 

The \texttt{PRM} in Algorithm~\ref{Alg:prm-raw} is simply a truncation-based Laplace mechanism. That is, it first truncates all the received data with the threshold $M$ (\textcolor{blue}{line 3}). Then, Laplace noise is added to the empirical mean to preserve privacy (\textcolor{blue}{line 5}). We highlight again that truncation here helps with both robustness (via removing outliers) and privacy (via bounding the sensitivity of empirical mean). 

As in the standard algorithm design of MABs, the key is to utilize the concentration of the mean estimator. To this end, we first give the following high-probability concentration result for the mean estimate returned by \texttt{PRM} in Algorithm~\ref{Alg:prm-raw}.

\begin{theorem}[Concentration of Mean Estimate]
\label{Thm:MeanRaw}
Given a collection of Huber-contaminated data $\{x_i\}_{i=1}^n$ where the inlier distribution satisfies Definition~\ref{def:raw} with mean $\mu$,  let $\widetilde{\mu}$ be the mean estimate by Algorithm~\ref{Alg:prm-raw}. Then, for any privacy budget $\epsilon >0$ and $\delta \in (0,1)$, the following results hold:

\textbf{Uncontaminated case.} For $\alpha = 0$, we have 
\begin{align*}
    |\widetilde{\mu}-\mu| = O\left(  \sqrt{\frac{\log(1/\delta)}{n}} + \frac{M \log(1/\delta)}{n\epsilon} + \frac{1}{M^{k-1}}\right),
\end{align*}
with probability at least $1-\delta$. Thus, choosing the truncation threshold $M=\Theta\left(\frac{n\epsilon}{\log(1/\delta)}\right)^{\frac{1}{k}}$ yields
\begin{align*}
    |\widetilde{\mu}-\mu|  = O\left(\sqrt{\frac{\log(1/\delta)}{n}} + \left(\frac{ \log(1/\delta)}{n\epsilon}\right)^{1-\frac{1}{k}}\right).
\end{align*}
\textbf{Contaminated case.} For $\alpha \in (0,1]$ and $n = \Omega\left( \frac{\log(1/\delta)}{\alpha}\right)$, we have the following with probability at least $1-\delta$
\begin{align*}
   |\widetilde{\mu}-\mu| \!=\! O\left(\sqrt{\frac{\log(1/\delta)}{n}}\!+\!\frac{M \log(1/\delta)}{n\epsilon}\!+\!\frac{1}{M^{k-1}}\!+\!\alpha M\right). 
\end{align*} 
 Therefore, choosing the truncation threshold $M =\Theta \left( \min\left\{\left(\frac{n\epsilon}{\log(1/\delta)}\right)^{\frac{1}{k}},\alpha^{-\frac{1}{k}}\right\} \right)$, yields $|\widetilde{\mu}-\mu| \le \beta$, where
\begin{align*}
   \beta  = O\left(\sqrt{\frac{\log(1/\delta)}{n}} + \left(\frac{ \log(1/\delta)}{n\epsilon}\right)^{1-\frac{1}{k}} + \alpha^{1-\frac{1}{k}}\right).
\end{align*}
\end{theorem}

With the above result, several remarks are ready. \emph{First,} for the uncontaminated case, our concentration result consists of the standard sub-Gaussian term and a new one due to privacy and heavy-tailed data. It can be translated into a sample complexity bound, i.e., to guarantee $|\widetilde{\mu} - \mu| \le \eta$ for any $\eta \in (0,1)$, it requires the sample size to be $n \ge O(\frac{\log(1/\delta)}{\eta^2} + \frac{\log(1/\delta)}{\epsilon \eta^{\frac{k}{k-1}}})$, which is optimal since it matches the lower bound for private heavy-tail mean estimation (cf. Theorem 7.2 in~\citet{hopkins2022efficient}). \emph{Second,} for the contaminated case, it has an additional bias term $O(\alpha^{1-1/k})$, which is also known to be information theoretically optimal~\cite{diakonikolas2018algorithmic}. Thus, via truncation, the \texttt{PRM} given by Algorithm~\ref{Alg:prm-raw} achieves the optimal trade-off between accuracy, privacy and robustness, which in turn shows its potential to be integrated into our meta-algorithm.



Now, based on the concentration result, we can set other missing parameters in our meta-algorithm accordingly. In particular, we have the following theorem that states the specific instantiation along with its performance guarantees. 
\begin{theorem}[Performance Guarantees]
\label{thm:RegRawUp}
Consider a private and robust MAB with inlier distributions satisfying Definition~\ref{def:raw} and $\alpha \in (0,1]$.
Let Algorithm~\ref{Alg:meta} be instantiated with Algorithm~\ref{Alg:prm-raw} and $M_{\tau}$, $\beta_{\tau}$ be given by Theorem~\ref{Thm:MeanRaw} with $n$ replaced by $B_{\tau}$. Set $\cT = \Omega(\frac{\log(1/\delta)}{\alpha})$ and $\delta = 1/T$. Then Algorithm~\ref{Alg:meta} is $\epsilon$-DP with its regret upper bound
\begin{align*}
    \!\!O\left(\sqrt{KT \log T}\!+\!\left(\frac{K\log T}{\epsilon}\right)^{\frac{k-1}{k}}T^{\frac{1}{k}}\!+\!T \alpha^{1-\frac{1}{k}}\!+\!\frac{K \log T}{\alpha}\right).
\end{align*}
\end{theorem}

The above theorem presents the first achievable regret guarantee for private and robust bandits. 
The first three terms match our lower bound in Theorem~\ref{thm: lowerBound} up to $\log T$ factor. The last additive term is mainly due to the fact that the mean concentration result only holds when the sample size is larger than $\cT = \Omega( \frac{\log(1/\delta)}{\alpha})$. As a result, each sub-optimal has to be played at least $\Omega( \frac{\log(1/\delta)}{\alpha})$ times. However, for a sufficiently large $T$ and a constant $\alpha$, the last term is dominated by other terms.

\begin{remark}
    For the case when $\alpha = 0$, using the uncontaminated concentration bound in Theorem~\ref{Thm:MeanRaw} and the same analysis, we achieve a regret upper bound  $O(\sqrt{KT \log T}\!+\!(\frac{K\log T}{\epsilon})^{\frac{k-1}{k}}T^{\frac{1}{k}}\!)$, which also matches the lower bound up to $\log T$ factor.
\end{remark}

\subsection{Finite Central Moment Case}
\label{sec:CentralUp}
The setting in the last section for the finite raw moment case may not be entirely satisfactory as it essentially assumes that the mean of arms is bounded within a small range (hence the sub-optimal gaps). Thus, in this section, we turn to private and robust MABs where the inlier distributions have a finite $k$-th central moment as given by Definition~\ref{def:central}. To this end, we first need a new \texttt{PRM},  since now simply truncating around zero as in Algorithm~\ref{Alg:prm-raw} will not work.

\begin{algorithm}[t!]
\caption{\texttt{PRM} for the finite central moment case}
\label{alg:prm-central}
\begin{algorithmic}[1]
 \STATE {\bfseries Input:} A collection of data $\{x_i\}_{i=1}^{2n}$, truncation parameter $M$, additional parameters $\Phi = \{\epsilon, D, r\}$, $r \in \mathbb{R}$.
   \STATE{\textcolor{blue}{// First step: initial estimate}}
 \STATE $B_j=[j,j+r), j \in \mathcal{J}=\{-D,-D+r,\dots,D-r\}$
\STATE Compute private histogram  using the first fold of data:
$$\widetilde{p}_j=\frac{\sum_{i=1}^n \mathbbm{1}_{\{X_i \in B_j\}}}{n}+\operatorname{Lap}\left(\frac{2}{n\epsilon}\right)$$
\STATE Get the initial estimate $J=\arg\max_{j\in \mathcal{J}} \widetilde{p}_j$ 
 \STATE{\textcolor{blue}{// Second step: final estimate}}
\STATE Get final estimator using the second fold of data: $\widetilde{\mu}= J + \frac{1}{n}\sum_{i=n+1}^{2n} (X_i-J)\mathbbm{1}_{\{|X_i-J|\le M\}}+ \operatorname{Lap}\left(\frac{2M}{n\epsilon}\right)$
\end{algorithmic}
\label{Algo:HistCDP2}
\end{algorithm}

Our new \texttt{PRM} is presented in Algorithm~\ref{alg:prm-central}, which consists of two steps. The intuition is simple: the first step aims to have a rough estimate of the mean, which is necessary since now the mean could be far away from zero. Then, in the second step, it truncates around the initial estimate to return the final result. More specifically, in the first step, we first construct bins over the range $[-D, D]$, which is assumed to contain the true mean by Definition~\ref{def:central}. Then, we compute the private histogram via the Laplace mechanism. The initial estimate is given by the left endpoint of the bin that has the largest empirical mass. Next, in the second step, it simply truncates around the initial estimate and again adds Laplace noise for privacy. 

\begin{remark}
    It is worth noting that a similar idea of two-step estimation has been used in previous work on robust mean estimation in the one-dimensional heavy-tailed case~\cite{prasad2019unified,kamath2020private, li2022robustness,}. However, there are several differences in our algorithm design and analysis. In particular, while~\cite{prasad2019unified} considers mean estimation under Huber's model without privacy constraints, we further impose differential privacy requirements. As a result, the estimates for both two steps are in different forms in our case compared to~\cite{prasad2019unified}, though they share the same high-level intuition. On the other hand, while~\cite{kamath2020private} considers mean estimation under differential privacy, there is no consideration of Huber contamination as in our case. Moreover, our second estimate is based on truncation while their method is via medians-of-means. In fact, as will be shown later (see Remark~\ref{rem:k20}), when our result reduces to the uncontaminated case, it achieves improvement over the one in~\cite{kamath2020private}. Finally,~\cite{li2022robustness}\footnote{In particular, we refer to the first arxiv version of~\cite{li2022robustness}. } considers both Huber contamination and \emph{local} differential privacy, and establishes the corresponding mean square error (MSE). In contrast, we consider the \emph{central} differential privacy and aim to establish a high-probability tail concentration. To this end, we take a different truncation method (i.e., using an indicator function in Line 7) compared to the one in~\cite{li2022robustness}.
\end{remark}

As before, we first present the concentration property of our new \texttt{PRM}, which will manifest in the specific instantiation of our meta-algorithm. In particular, we first give the following general theorem and then state two more detailed corollaries. 

\begin{theorem}[Concentration of Mean Estimate]
\label{thm:CentralMeanUp}
Given a collection of Huber-contaminated data $\{x_i\}_{i=1}^{2n}$ where the inlier distribution satisfies Definition~\ref{def:central} with mean $\mu$,  let $\widetilde{\mu}$ be the mean estimate by Algorithm~\ref{alg:prm-central}.    
    For any $\alpha \in (0,\alpha_{\max})$, $\epsilon \in (0,1]$ and $\delta \in (0,1)$, there exist some constants $\mathcal{T}(\alpha,\epsilon,\delta)$, $r$, $M$ and  $D \ge 2r$ such that for all $n \ge \mathcal{T}(\alpha,\epsilon,\delta)$, with probability at least $1-\delta$
    \begin{align*}
          \!|\widetilde{\mu}-\mu|\!\le\! O\left(\sqrt{\frac{\log(1/\delta)}{n}}\!+\!\frac{M\log(1/\delta)}{n\epsilon}+\frac{1}{M^{k-1}}\!+\!\alpha M\right),
    \end{align*}
    where $\alpha_{\max} < 1$ is the breakdown point.
\end{theorem}


The above theorem follows the same pattern as the one for the raw moment case (Theorem~\ref{Thm:MeanRaw}). The key differences are the threshold value $\mathcal{T}(\alpha,\epsilon,\delta)$ and  the breakdown point $\alpha_{\max}$, which are summarized in the following results.

\begin{corollary}[Mean Concentration, $\alpha = 0$] 
\label{cor:clean}
Let the same assumptions in Theorem~\ref{thm:CentralMeanUp} hold. 
    For any $\epsilon \in (0,1]$, setting $r = 10^{1/k}$ and $M = \Theta\left(\frac{n\epsilon}{\log(1/\delta)}\right)^{1/k}$, then for all $n \ge \Omega\left(\log(D/\delta)/\epsilon\right)$ and $D \ge 2r$,  we have that for any $\delta \in (0,1)$, with probability at least $1-\delta$, $ |\tilde{\mu} - \mu| \le \beta$ where
    \begin{align*}
   \beta = O\left(\sqrt{\frac{\log(1/\delta)}{n}}+\left(\frac{\log(1/\delta)}{n\epsilon}\right)^{1-\frac{1}{k}}\right).
\end{align*}
In other words, taking number of samples $n$ such that 
\begin{align*}
    n \ge O\left(\frac{\log(1/\delta)}{\eta^2} + \frac{\log(1/\delta)}{\epsilon \eta^{\frac{k}{k-1}}} + \frac{\log(D/\delta)}{\epsilon}\right),
\end{align*}
we have $|\tilde{\mu} - \mu| \le \eta$ with probability at least $1-\delta$.
\end{corollary}
\begin{remark}
\label{rem:k20}
    The above lemma strictly improves the result in~\cite[Theorem 3.5]{kamath2020private}\footnote{We also note that the main focus of \citet{kamath2020private} is not on achieving the optimal estimate.}. In particular, it uses the method of medians-of-means and achieves $\frac{\log D \cdot \log(1/\delta)}{\epsilon}$ for the third term. In contrast, our third term is additive rather than multiplicative. In fact, our concentration is optimal, which matches the lower bound for the one-dimensional case (see~\cite[Theorem 7.2]{hopkins2022efficient}).
\end{remark}

\begin{corollary}[Mean Concentration, $\alpha > 0$]
\label{cor:con}
Let the same assumptions in Theorem~\ref{thm:CentralMeanUp} hold. 
    For any $\epsilon \in (0,1]$ and $\alpha \in (0,0.133)$, we let $r = \iota^{1/k}$ where $\iota = \frac{1-\alpha}{0.249 - \alpha}$ and $M =\Theta( \min\{\left(\frac{n\epsilon}{\log(1/\delta)}\right)^{1/k}, ( \alpha)^{-1/k}\})$. Then, there exists constant $c_1$, for all $n$ such that $n \ge \cT = \Omega(\max\{\frac{\iota\log(1/\delta)}{\epsilon}, \frac{c_1\log(D/\delta)}{\epsilon}, \frac{\log(1/\delta)}{\alpha^2}\})$ and $D \ge 2r$, we have that for any $\delta \in (0,1)$, with probability at least $1-\delta$, $|\widetilde{\mu} - \mu| \le \beta$ with 
    \begin{align*}
   \beta = O\left(\sqrt{\frac{\log(1/\delta)}{n}}+\left(\frac{\log(1/\delta)}{n\epsilon}\right)^{1-\frac{1}{k}} + \alpha^{1-\frac{1}{k}}\right).
\end{align*}
\end{corollary}
\begin{remark}
    The above concentration has the same form as the one in Theorem~\ref{Thm:MeanRaw}. Specifically, for a large sample size $n$, it has the optimal concentration (for small $\alpha$). The threshold $\cT$ on $n$ depends on both $\alpha, \epsilon$ now. We note that even for the sub-Gaussian inlier distributions without privacy protection, the existing concentration also has a threshold $\cT = \frac{\log(1/\delta)}{\alpha^2}$ (see Lemma 4.1 in~\citet{mukherjee2021mean}). 
\end{remark}

Now, we are left to leverage the above two concentration results to design specific  instantiations of our meta-algorithm and establish their performance guarantees.

Our first instantiation is for the uncontaminated case, i.e., $\alpha = 0$. Therefore, robustness is then only with respect to heavy-tailed rewards while privacy is still preserved.

\begin{theorem}[Performance Guarantees, $\alpha = 0$]
\label{thm:central-clean}
Consider a private and robust MAB with inlier distributions satisfying Definition~\ref{def:raw} and $\alpha = 0$.
Let Algorithm~\ref{Alg:meta} be instantiated with Algorithm~\ref{alg:prm-central}, and $r$, $M_{\tau}$, $\beta_{\tau}$ be given by Corollary~\ref{cor:clean} with $n$ replaced by $B_{\tau}$. Set $\cT = \Omega(\frac{\log(D/\delta)}{\epsilon})$ and $\delta = 1/T$. Then, Algorithm~\ref{Alg:meta} is $\epsilon$-DP with its regret upper bound
\begin{align*}
    \!\!O\left(\sqrt{KT \log T}\!+\!\left(\frac{K\log T}{\epsilon}\right)^{\frac{k-1}{k}}T^{\frac{1}{k}}\!+\!\gamma \right),
\end{align*}
where $\gamma:= O\left(\frac{KD \log (DT)}{\epsilon}\right)$.
\end{theorem}
\begin{remark}
    To the best of our knowledge, this is the first result on private and heavy-tailed bandits with the finite central moment assumption. The state-of-the-art result is only focused on the simpler case, i.e., the finite raw moment assumption~\cite{tao2021optimal}.  
\end{remark}

Finally, armed with Corollary~\ref{cor:con}, we have the second instantiation of our meta-algorithm that deals with the contaminated case. 
\begin{theorem}[Performance Guarantees, $\alpha > 0$]
\label{thm:central-cont}
Consider a private and robust MAB with inlier distributions satisfying Definition~\ref{def:raw} and $\alpha \in (0,0.133)$.
Let Algorithm~\ref{Alg:meta} be instantiated with Algorithm~\ref{alg:prm-central}, and $r$, $\cT$, $M_{\tau}$, $\beta_{\tau}$ be given by Corollary~\ref{cor:con} with $n$ replaced by $B_{\tau}$. Set $\delta = 1/T$, then Algorithm~\ref{Alg:meta} is $\epsilon$-DP with its regret upper bound
\begin{align*}
    \!\!O\left(\sqrt{KT \log T}\!+\!\left(\frac{K\log T}{\epsilon}\right)^{\frac{k-1}{k}}T^{\frac{1}{k}}\!+\!T \alpha^{1-\frac{1}{k}}\!+\!\hat{\gamma} \right),
\end{align*}
where $\hat{\gamma}:= O\left(\frac{D K \log T}{\alpha^2}+\frac{\iota D K\log T}{\epsilon}+\frac{ D K\log(DT)}{\epsilon}\right)$ and $\iota = \frac{1-\alpha}{0.249 - \alpha}$.
\end{theorem}
The above upper bound also matches our lower bound up to $O(\hat{\gamma})$, which is dominated by other terms for a sufficiently large $T$ and constant $\alpha, D$.

\section{Experiments}
\begin{figure*}[!ht] 
	\centering  
	\subfigure[Student's $t$, $\alpha=5\%$, $\epsilon=0.2$]{
		\label{fig-p-10-0.5}
		\includegraphics[width=0.31\linewidth]{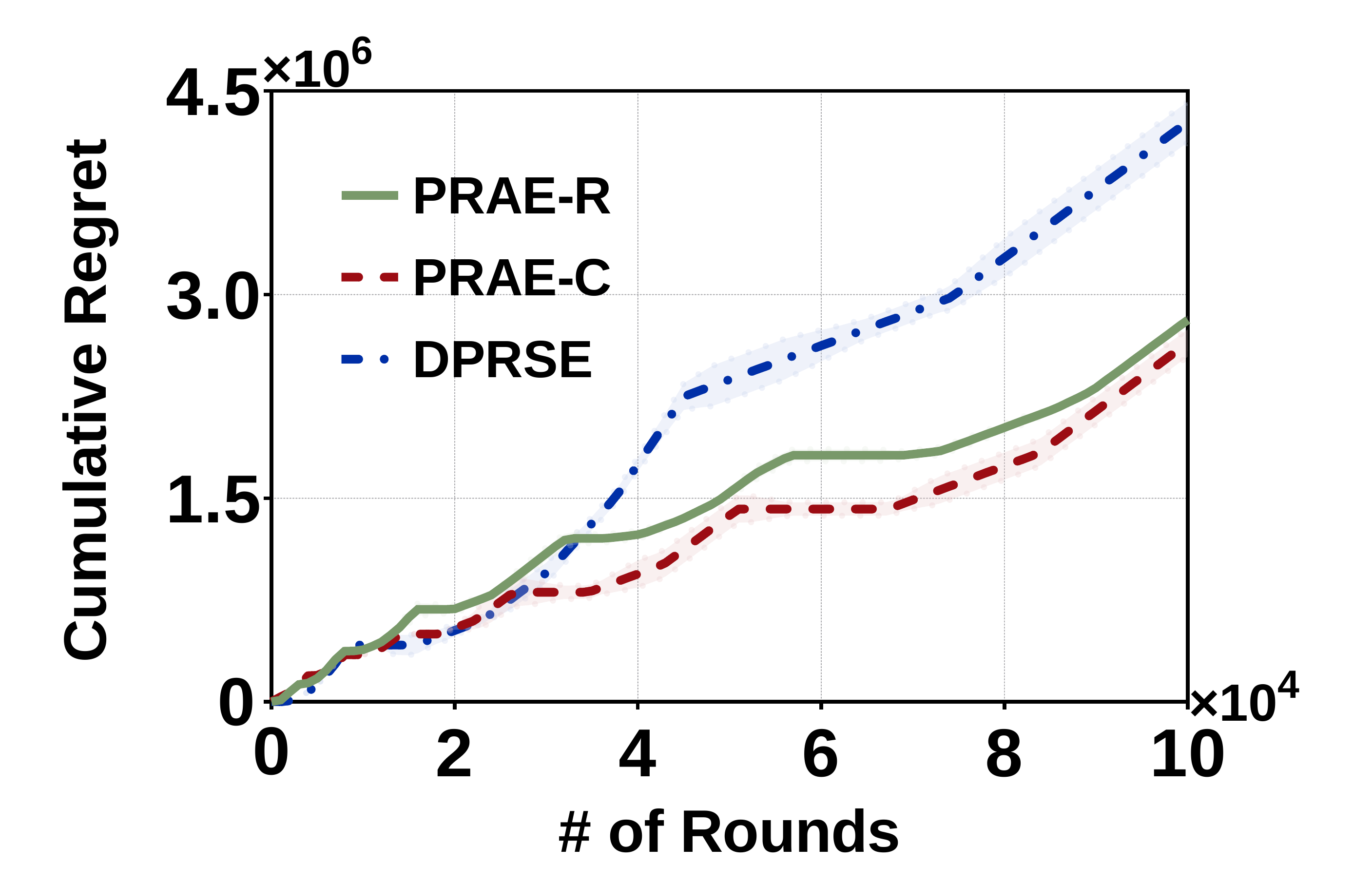}}
    \subfigure[Student's $t$, $\alpha=5\%$, $\epsilon=0.5$]{
		\label{fig-p-10-1.0}
		\includegraphics[width=0.31\linewidth]{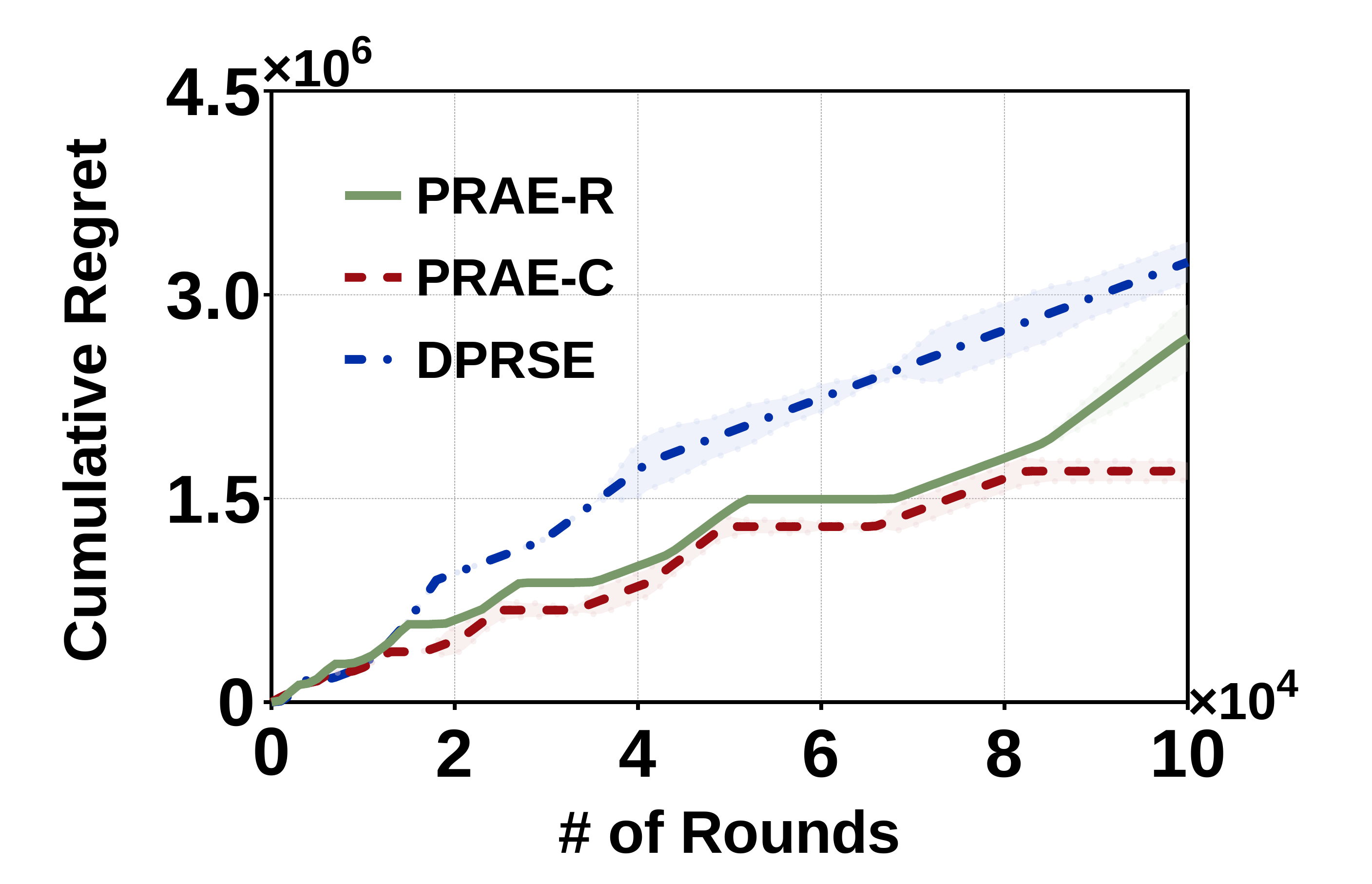}}
	\subfigure[Student's $t$, $\alpha=5\%$, $\epsilon=1.0$]{
		\label{fig-p-10-10.0}
		\includegraphics[width=0.31\linewidth]{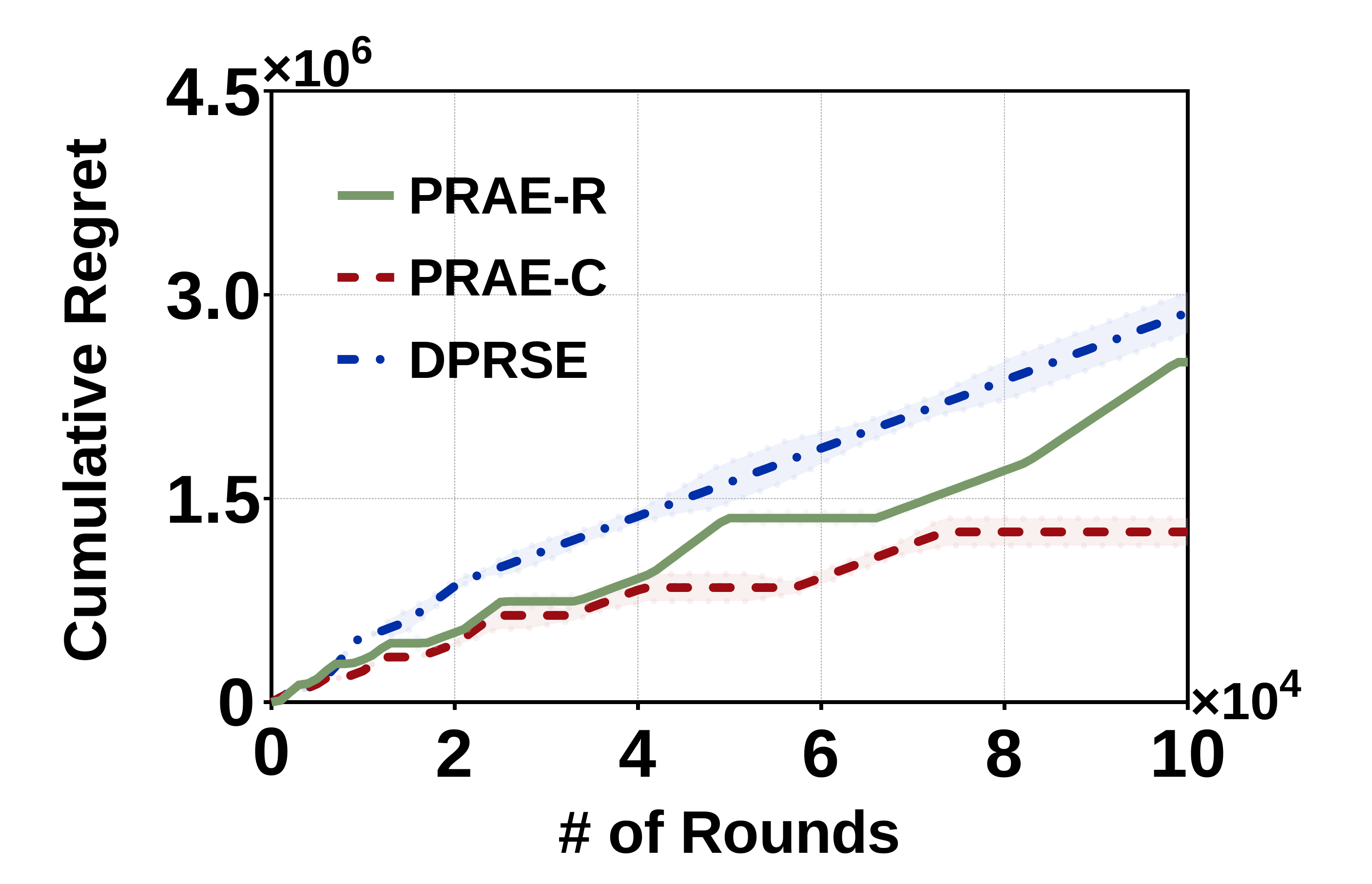}}

    \subfigure[Pareto, $\alpha=2\%$, $\epsilon=0.5$]{
		\label{fig-t-5-1.0}
		\includegraphics[width=0.31\linewidth]{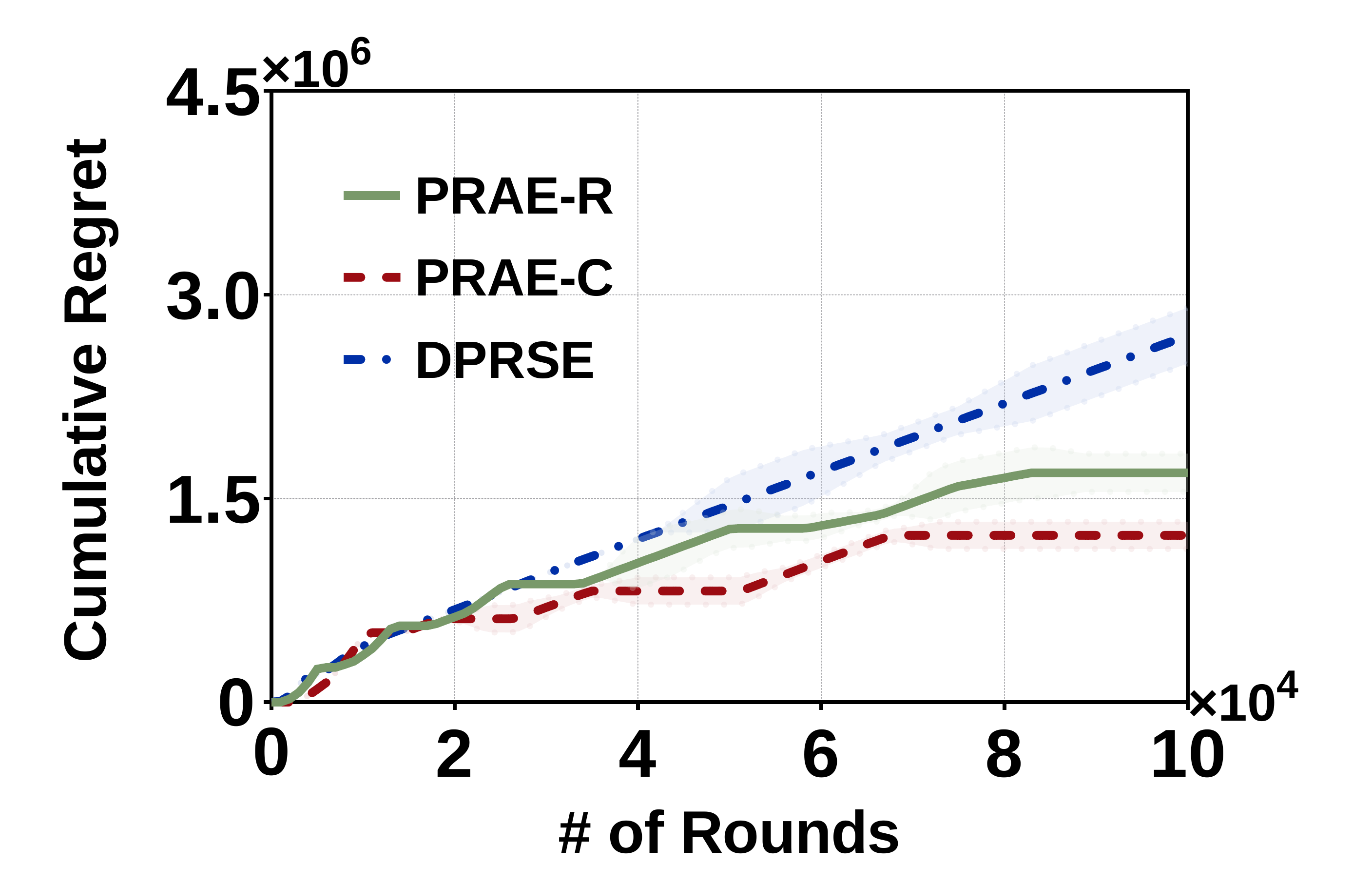}}
    \subfigure[Pareto, $\alpha=5\%$, $\epsilon=0.5$]{
		\label{fig-t-10-1.0}
		\includegraphics[width=0.31\linewidth]{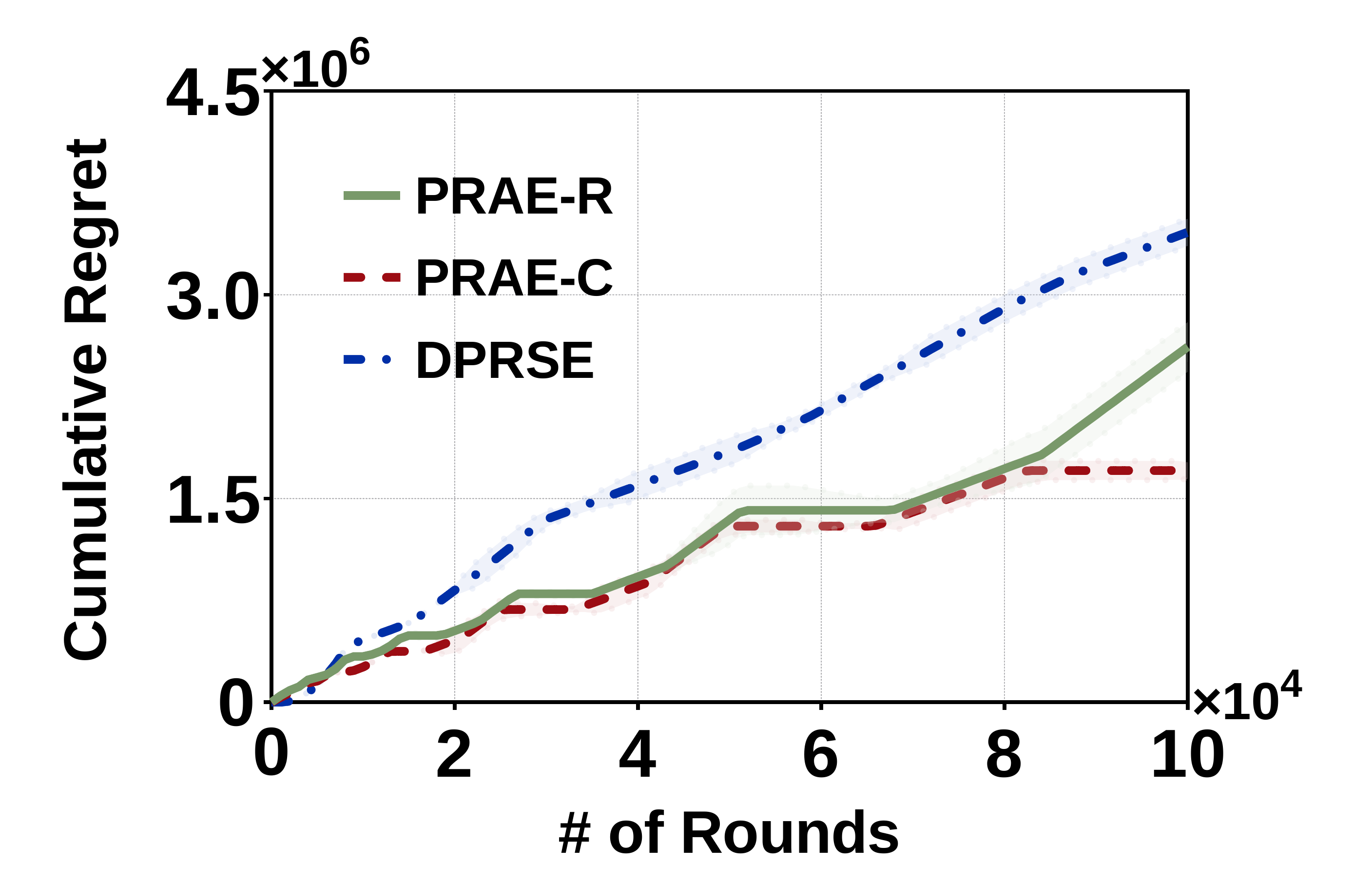}}
	\subfigure[Pareto, $\alpha=10\%$, $\epsilon=0.5$]{
		\label{fig-t-20-1.0}
		\includegraphics[width=0.31\linewidth]{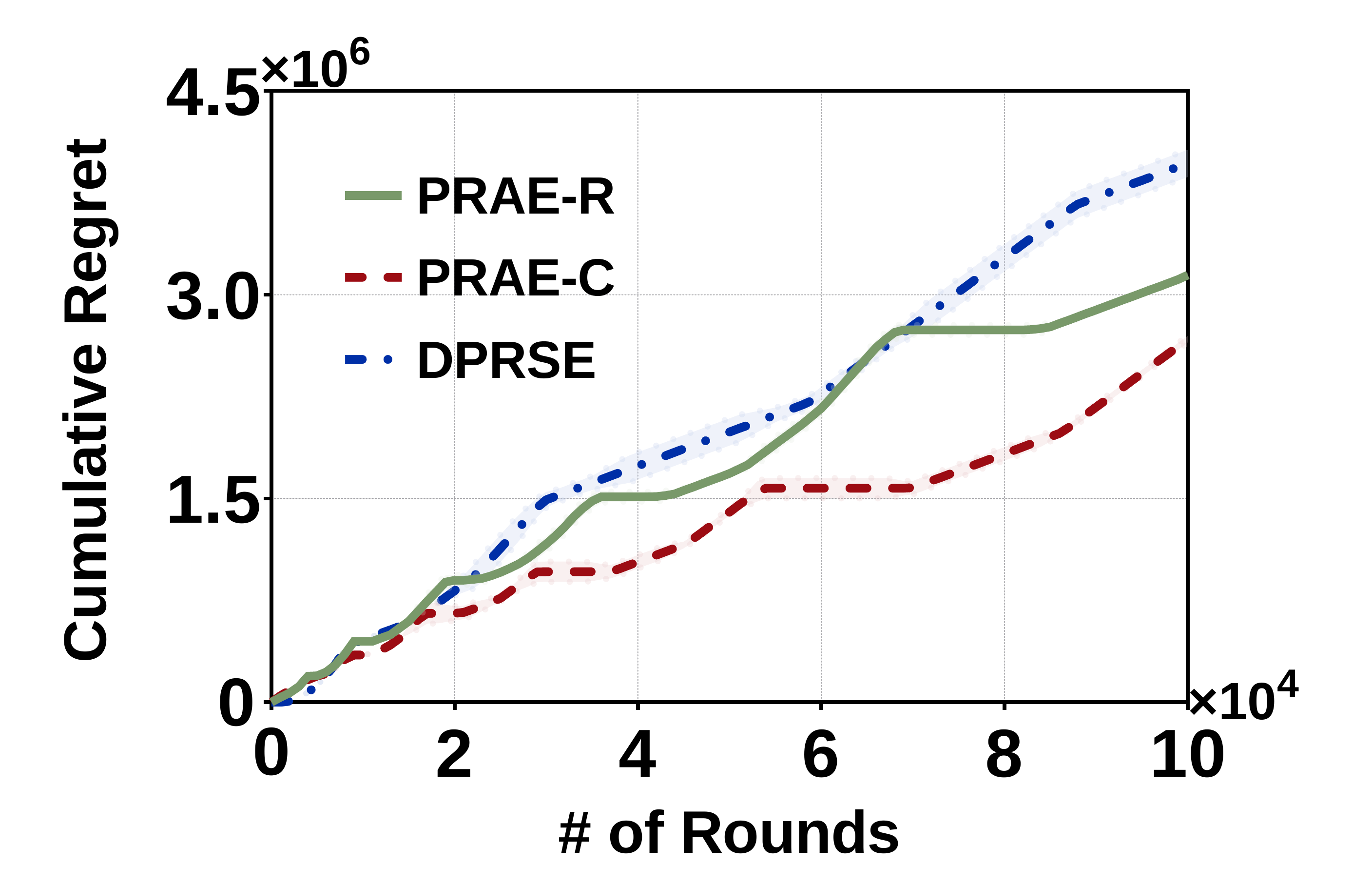}}
	\caption{Comparison of cumulative regret for PRAE-R, PRAE-C and DPRSE. \textbf{Top:} Rewards generated from Student's $t$-distribution with fixed Huber parameter $\alpha=10\%$ and varying privacy budget $\epsilon\in\{0.2, 0.5, 1.0\}$. \textbf{Bottom:} Rewards generated from Pareto distribution with fixed privacy budget $\epsilon=1.0$ and varying Huber parameter $\alpha\in\{2\%, 5\%, 10\%\}$.}
	\label{Fig-exp}
\end{figure*}

In this section, we will empirically evaluate the practical performance of our private and robust arm elimination algorithms, which are abbreviated as PRAE-R and PRAE-C when the sub-routine \texttt{PRM} is Algorithm \ref{Alg:prm-raw} for the finite raw moment case and Algorithm \ref{alg:prm-central} for the finite central moment case, respectively. We compare them with the DPRSE algorithm in \citet{tao2021optimal}, which achieves the optimal regret bound for DP heavy-tailed MAB.

\subsection{Experiment Setup}
We consider the case where there are $K=5$ arms, and the mean of each arm is within the range of $[0,100]$. Specifically, we let the arm means descend linearly, i.e., for each arm $a\in[k]$, let $\mu_a= 100-\frac{100(a-1)}{K-1}$. We consider the following two types of heavy-tailed distributions for the true inlier reward  generation:

{\it{\Large -} Pareto distribution:} For each pull of arm $a$, we generate a reward 
that is sampled from the distribution $\mu_a+\eta-2.5$, where $\eta\sim\frac{sx_m^s}{x^{s+1}}\mathbbm{1}_{\{x\ge x_m\}}$ for $x\in\mathbb{R}$ and we set the shape parameter $s=2.5$ and the scale parameter $x_m=1.5$.

{\it{\Large -} Student’s $t$-distribution:} For each pull of arm $a$ we generate a reward that is sampled from the distribution $\mu_a+\eta$, where $\eta\sim\frac{\Gamma(\frac{\nu+1}{2})}{\sqrt{\nu\pi}\Gamma(\frac{\nu}{2})}\Big(1+\frac{x^2}{\nu}\Big)^{-\frac{\nu+1}{2}}$. Here we set the degree of freedom $\nu=2.5$.
For both cases, the stochastic rewards have finite second central moment of $5$. The main difference between the above two types of distribution is that the Student’s $t$-distribution is symmetric while the Pareto distribution is one-sided. To generate contaminated rewards, we consider Gaussian distribution with zero mean and standard deviation of $50$.

For Student's $t$-distribution case, we will fix the Huber parameter $\alpha=5\%$ with different privacy budget $\epsilon\in \{0.2, 0.5, 1\}$, while for the Parote distribution case we fix  $\epsilon=0.5$ and vary the corruption level $\alpha=\{2\%, 5\%, 10\%\}$. For each experiment, we repeat 30 times and set the total number of round $T=10^5$ (thus we set $\delta=10^{-5}$). We will report the average of cumulative regrets $\mathcal{R}_T$  with respect to the number of rounds. 

\vspace{-0.1in}
\subsection{Results and Discussions}

We present our results in Figure~\ref{Fig-exp}. 
For all cases, PRAE-R and PRAE-C achieve smaller cumulative regret and thus better expected performance than DPRSE. More specifically, 
from Fig.~\ref{fig-t-5-1.0}-\ref{fig-t-20-1.0}, we can see that, when the Huber parameter $\alpha$  increases, DPRSE diverges to a larger regret, while PRAE-R and PRAE-C are only limitedly affected. This is because DPRSE adopts  more aggressive truncation thresholds which incorporate more outliers. In contrast, the truncation thresholds in PRAE-R and PRAE-C are delicately designed and thus provide robustness against contaminated rewards. In addition, by observing the error bars, we find all three methods are stable under both symmetric and one-sided types of heavy-tailed distribution (see, e.g., Fig,~\ref{fig-p-10-1.0} and~\ref{fig-t-10-1.0}).  Thus, we can conclude that our approaches PRAE-R and PRAE-C outperform the baseline method DPRSE.

For both PRAE-R and PRAE-C, we can also observe that when $\epsilon$ is smaller or $\alpha$ is larger, the regret will increase  for both types of distributions, which is due to the fact that the regret bound is proportional to $1/\epsilon$ and $\alpha$ when $T$ is large enough. Moreover, compared with PRAE-R, we can see the regret of PRAE-C is lower for all experiments. This is due to the fact that the \texttt{PRM} subroutine for PRAE-C leverages the prior information (i.e., range $D$) of  mean for each arm, which could provide finer performance bound for Algorithm \ref{Alg:meta}. In total, all above results corroborate our theories. 

\vspace{-0.1in}
\section{Conclusion}
In this paper, we investigated private and robust multi-armed bandits with heavy-tailed rewards under Huber's contamination model as well as differential privacy constraints.  We proposed a meta-algorithm that builds on a private and robust mean estimation sub-routine \texttt{PRM}. For two different heavy-tailed settings, we
provided specific schemes of \texttt{PRM}, both of which only rely on the truncation and the Laplace mechanism. Moreover, we also established regret upper bounds for these algorithms, which nearly match our derived minimax lower bound. We also conducted experiments to support our theoretical analysis. 

\section{Acknowledgments}
 YW and DW are supported in part by  BAS/1/1689-01-01, URF/1/4663-01-01, FCC/1/1976-49-01 of King Abdullah University of Science and Technology. XZ is supported in part by NSF CNS-2153220. We thank Mengchu Li for the insightful discussions  and for pointing out the first arxiv version of~\cite{li2022robustness}.



\printbibliography


\newpage
\appendix
\onecolumn
\section{Useful Lemmas}

\begin{lemma}[Post-Processing \cite{dwork2014algorithmic}] 
\label{PostPro}
 Let $\mathcal{M}: \mathcal{X}\rightarrow \mathcal{Y}$ be a randomized algorithm that is $(\epsilon,\delta)$-differentially private. Let $f: \mathcal{Y} \rightarrow \mathcal{Z}$ be an arbitrary randomize mapping. Then $f \circ \mathcal{M}: \mathcal{X}\rightarrow \mathcal{Z}$ is $(\epsilon,\delta)$-differentially private.
 \end{lemma}

 \begin{lemma}[Composition Theorem \cite{dwork2014algorithmic}]
 \label{compositionThm}
 Let $\mathcal{M}_1,\mathcal{M}_2,\dots,\mathcal{M}_h$ be a sequence of randomized algorithms, where $\mathcal{M}_1: \mathcal{X}^n \rightarrow \mathcal{Y}_1$, $\mathcal{M}_2: \mathcal{Y}_1 \times \mathcal{X}^n \rightarrow \mathcal{Y}_2,\dots$, $\mathcal{M}_h: \mathcal{Y}_1 \times \mathcal{Y}_2 \times \mathcal{Y}_{h-1} \times \mathcal{X}^n \rightarrow \mathcal{Y}_h$. Suppose for every $i \in [h]$ and $y_1 \in \mathcal{Y}_1,y_2 \in \mathcal{Y}_2,\dots,y_h \in \mathcal{Y}_h$, we have $\mathcal{M}_i(y_1,\dots,y_{i-1},\cdot) : \mathcal{X}^n \rightarrow \mathcal{Y}_i$ is $\epsilon_i$-DP. Then the algorithm $\mathcal{M}: \mathcal{X}^n \rightarrow \mathcal{Y}_1 \times  \mathcal{Y}_2 \times  \dots \mathcal{Y}_h$ that runs the algorithm $\mathcal{M}_i$ sequentially is $\epsilon$-DP for $\epsilon=\sum_{i=1}^h \epsilon_i$.
\end{lemma}

\begin{lemma}[Parallel Composition \cite{near2021programming}]
\label{lem:parallel}
    Suppose there are $n$ $\epsilon$-differentially private mechanisms $\{\mathcal{M}_i\}_{i=1}^n$ and $n$ disjoint datasets denoted by $\{D_i\}_{i=1}^n$. Then for the algorithm which applies each $\mathcal{M}_i$ on the corresponding $D_i$, it is $\epsilon$-DP.
\end{lemma}

\begin{lemma}[Markov's inequality]
\label{lem:Markov}
If $Y \in \mathbb{R}$ is a random variable and $a>0$, we have
    $$\mathbb{P}\left(|Y| \geq a\right) \leq \frac{\mathbb{E}\left(|Y|^k\right)}{a^k}$$
\end{lemma}

\begin{lemma}[Chebyshev’s inequality]
\label{lem:ChebIneq}
    For a real-valued random variable $Y \in \mathbb{R}$, $a>0$ and $k \in \mathbb{N}$, we have
    $$
\mathbb{P}(|Y-\mathbb{E} Y| \geq a)=\mathbb{P}\left(|Y-\mathbb{E} Y|^k \geq a^k\right) \leq \frac{\mathbb{E}\left(|Y-\mathbb{E} Y|^k\right)}{a^k}
$$
\end{lemma}
\begin{lemma}[Tail Bound of Laplacian Vairable \cite{dwork2006calibrating}]\label{lem:TailLap}
    If $X\sim{{\rm Lap}(b)}$, then
    \begin{equation*}
        \mathbb{P}(|X|\ge t\cdot b)=\exp(-t).
    \end{equation*}
\end{lemma}
\begin{lemma}[Hoeffding’s inequality]
\label{lem:Hoef}
Let $Z_1,\dots,Z_n$ be independent bounded
random variables with $Z_i \in [a, b]$ for all $i$, where $-\infty<a <b<\infty$. Then
    $$
\mathbb{P}\left(\left|\frac{1}{n} \sum_{i=1}^n\left(Z_i-\mathbb{E}\left[Z_i\right]\right)\right| \geq t\right) \leq 2\exp \left(-\frac{2 n t^2}{(b-a)^2}\right)
$$
\end{lemma}
\begin{lemma}[H\" older’s Inequality]
\label{lem:Holder}
Let $X,Y$ be random variables over $\mathbb{R}$, and let $k > 1$. Then,
    $$
\mathbb{E}[|X Y|] \leq\left(\mathbb{E}\left[|X|^k\right]\right)^{\frac{1}{k}}\left(\mathbb{E}\left[|Y|^{\frac{k}{k-1}}\right]\right)^{\frac{k-1}{k}}
$$
\end{lemma}
\begin{lemma}[Bernstein's Inequality \cite{vershynin2018high}]\label{lem:Bern}
Let $X_1, \cdots X_n$ be $n$ independent zero-mean random variables. Suppose $|X_i|\leq M$ and $\mathbb{E}[X_i^2] \le s$ for all $i$. Then for any $t>0$, we have 
\begin{equation*}
    \mathbb{P}\left\{\left|\frac{1}{n}\sum_{i=1}^n X_i\right| \geq t\right\}\leq 2\exp\left(-\frac{\frac{1}{2}t^2n}{s+\frac{1}{3}Mt}\right)
\end{equation*}
\end{lemma}
\section{Proofs of Section \ref{sec:lower}}

\begin{lemma}[Upper Bound on KL-divergence for Bandits with $\epsilon$-DP \cite{azize2022privacy}]\label{lem:DP_KL}
If $\pi$ is a mechanism satisfying $\epsilon$-DP, then for two instances $\nu_1=(r_a: a \in [K])$ and $\nu_2=(r_a^\prime: a \in [K])$ we have
$$\text{KL}\left(\mathbb{P}_{\pi,\nu_1}^T \|\mathbb{P}_{\pi,\nu_2}^T\right) \le 6\epsilon \mathbb{E}_{\pi,\nu_1}\left[\sum_{t=1}^T \text{TV}(r_{a_t}\|r^\prime_{a_t})\right]$$
where $\text{TV}(r_{a}\|r^\prime_{a})$ is the total-variation distance between $r_a$ and $r^\prime_a$.
\end{lemma}

\begin{lemma}[Theorem 5.1 in \cite{chen2018robust}]\label{lem:HuberTV}
    Let $R_1$ and $R_2$ be two distributions on $\mathcal{X}$. If for some $\alpha \in [0,1]$, we have that $TV(R_1,R_1)=\frac{\alpha}{1-\alpha}$, then there exists two distributions on the same probability space $G_1$ and $G_2$ such that
    $$(1-\alpha)R_1+\alpha G_1=(1-\alpha)R_2+\alpha G_2.$$
\end{lemma}

\begin{proof}[\bf Proof of Theorem \ref{thm: lowerBound}]
Let $\Pi$ be the set of all polices and $\Pi^\epsilon$ be the set of all $\epsilon$-DP policies. We denote the environment corresponding to the set of $K$-Gaussian reward distributions with means $\mu \in \mathbb{R}^K$
the same variance $\sigma_k^2$ where the value of $\sigma_k$ is determined by $k$ to make the $k$-th raw moments of the distributions are bounded by $1$ as $\mathcal{E}_{\mathcal{N}}^K(\sigma_k) \triangleq\left\{\left(\mathcal{N}\left(\mu_i, \sigma_k^2\right)\right)_{i=1}^K: \mu=\left(\mu_1, \ldots, \mu_K\right) \in \mathbb{R}^K\right\}$. Since $\Pi^\epsilon \subset \Pi$, we can have that $$\mathcal{R}_{T}^{\operatorname{minimax}}(\pi,\nu)\geq \inf _{\pi \in \Pi} \sup _{\nu \in \mathcal{E}_{\mathcal{N}}^K(\sigma_k)} \operatorname{Reg}_T(\pi, \nu) \geq \Omega{( \sqrt{KT})}$$
where the last inequality is due to Theorem 15.2 in \citet{lattimore2020bandit}.

\textbf{Case 1: Uncontaminated case.} By the definition of minimax regret, we know that $\mathcal{R}^{\text{minimax}}_{\epsilon,\alpha}\ge \mathcal{R}^{\text{minimax}}_{\epsilon,0}$. Therefore, we first derive the lower bound of private bandits without contamination.

We consider two environments. In the first environment $\nu_1$, the optimal arm (denote by $a_1$) follows $$
r_{a_1}= \begin{cases}1 / \gamma & \text { with probability of } \frac{1}{2}\gamma^{k} \\ 0 & \text { with probability of } 1-\frac{1}{2}\gamma^{k}\end{cases}
$$
where $\gamma \in (0,1]$. We can verify $\mathbb{E}{[r_{a_1}]}=\frac{1}{2}\gamma^{k-1}$ and  $\mathbb{E}{[r_{a_1}^k]}=\frac{1}{2} \le 1$. Any other sub-optimal arm $a\neq a_1$ in $\nu_1$ follows the same reward distribution$$
r_{a}= \begin{cases}1 / \gamma & \text { with probability of } \frac{3}{10}\gamma^{k} \\ 0 & \text { with probability of } 1-\frac{3}{10}\gamma^{k}\end{cases}
$$
We can verify $\mathbb{E}{[r_{a}]}=\frac{3}{10}\gamma^{k-1}$ and  $\mathbb{E}{[r_{a_1}^k]}=\frac{3}{10} \le 1$. Then the gap of means between the optimal arm and sub-optimal arm is $\Delta=\frac{1}{5}\gamma^{k-1}$.

For algorithm $\pi$ and instance $\nu_1$, we denote $
    i={\arg \min} _{a\in \{2,\cdots,K\}}\mathbb{E}_{\pi,\nu_1}[N_a(T)].$ Thus, $\mathbb{E}_{\pi,\nu_1}[N_i(T)]\leq\frac{T}{K-1}$.

Now, consider another instance $\nu_2$ where $r_{a_1},\cdots,r_{a_k}$ are the same as those in $\nu_1$ except the $i$-th arm such that
$$
r_{i}^\prime= \begin{cases}1 / \gamma & \text { with probability of } \frac{7}{10}\gamma^{k} \\ 0 & \text { with probability of } 1-\frac{7}{10}\gamma^{k}\end{cases}
$$
We can verify $\mathbb{E}{[r_{i}^\prime]}=\frac{7}{10}\gamma^{k-1}$ and  $\mathbb{E}{[(r_{i}^\prime)^k]}=\frac{7}{10} \le 1$. Then in $\nu_2$, the arm $i$ is optimal.

Now by the classic regret decomposition, we obtain 
$$\mathcal{R}_{T}(\pi,\nu_1) =  (T-\mathbb{E}_{\pi,\nu_1}[N_1(T)])\Delta \ge \mathbb{P}_{\pi,\nu_1}^T \left[N_1(T) \le \frac{T}{2}\right]\frac{T\Delta}{2}. $$
$$\mathcal{R}_{T}(\pi,\nu_2) = \Delta \mathbb{E}_{\pi,\nu_2}[N_1(T)] + \sum_{a\notin\{1,i\}}2\Delta \mathbb{E}_{\pi,\nu_2}[N_a(T)] \ge \mathbb{P}_{\pi,\nu_2}^T \left[N_1(T) \ge \frac{T}{2}\right]\frac{T\Delta}{2}.$$
By applying the Bretagnolle–Huber inequality (\cite{lattimore2020bandit}, Theorem 14.2), we have 

$$\begin{aligned}
 \mathcal{R}_{T}(\pi,\nu_1)+ \mathcal{R}_{T}(\pi,\nu_2)
    & \ge \frac{T\Delta}{2}\left(\mathbb{P}_{\pi,\nu_1}^T \left[N_1(T) \le \frac{T}{2}\right]+\mathbb{P}_{\pi,\nu_2}^T \left[N_1(T) \ge \frac{T}{2}\right]\right). \\
    & \ge \frac{T\Delta}{4} \exp{\left(-\text{KL}\left(\mathbb{P}_{\pi,\nu_1}^T \|\mathbb{P}_{\pi,\nu_2}^T\right)\right)}
\end{aligned}$$

Based on Lemma \ref{lem:DP_KL}, we can get the upper bound of the KL-Divergence between the marginals.
$$
\begin{aligned}
   \text{KL}\left(\mathbb{P}_{\pi,\nu_1}^T \|\mathbb{P}_{\pi,\nu_2}^T\right) & \le 6\epsilon \mathbb{E}_{\pi,\nu_1}\left[\sum_{t=1}^T \text{TV}(r_{a_t}\|r^\prime_{a_t})\right] \\
   & \le 6\epsilon \mathbb{E}_{\pi,\nu_1}[N_i(T)]\text{TV}(r_i\|r^\prime_{i})
\end{aligned}
$$
since $\nu_1$ and $\nu_i$ only differ in the arm $i$.

Thus, $$\begin{aligned}
     \mathcal{R}_{T}(\pi,\nu_1)+ \mathcal{R}_{T}(\pi,\nu_2)
    & \ge \frac{T\Delta}{4} \exp{(-6\epsilon \mathbb{E}_{\pi,\nu_1}[N_i(T)] \cdot \frac{2}{5}\gamma^k)}\\
    & \ge \frac{T\gamma^{k-1}}{20} \exp{\left(-\frac{12\cdot  \epsilon T \gamma^k}{5(K-1)}\right)}.
\end{aligned}$$
Taking $\gamma= \left(\frac{K-1}{T\epsilon}\right)^{\frac{1}{k}}$, we get the result
$$\mathcal{R}_{T}(\pi,\nu_1)\ge \Omega\left(\left(\frac{K}{\epsilon}\right)^{\frac{k-1}{k}}T^{\frac{1}{k}}\right).$$

\textbf{Case 2: Contaminated case.} For $\alpha \neq 0$ and $\alpha \in (0,1]$, we still consider the true distributions of arms are the same in above $\nu_1$ and $\nu_2$. In the first environment $\nu_1$, the optimal arm (denote by $a_1$) follows $$
r_{a_1}= \begin{cases}1 / \gamma & \text { with probability of } \frac{1}{2}\gamma^{k} \\ 0 & \text { with probability of } 1-\frac{1}{2}\gamma^{k}\end{cases}
$$
where $\gamma \in (0,1]$. We can verify $\mathbb{E}{[r_{a_1}]}=\frac{1}{2}\gamma^{k-1}$ and  $\mathbb{E}{[r_{a_1}^k]}=\frac{1}{2} \le 1$. Any other sub-optimal arm $a\neq a_1$ in $\nu_1$ follows the same reward distribution$$
r_{a}= \begin{cases}1 / \gamma & \text { with probability of } \frac{3}{10}\gamma^{k} \\ 0 & \text { with probability of } 1-\frac{3}{10}\gamma^{k}\end{cases}
$$
We can verify $\mathbb{E}{[r_{a}]}=\frac{3}{10}\gamma^{k-1}$ and  $\mathbb{E}{[r_{a_1}^k]}=\frac{3}{10} \le 1$. Then the gap of means between the optimal arm and sub-optimal arm is $\Delta=\frac{1}{5}\gamma^{k-1}$.

And we denote the contaminated version of $\nu_1$ as $\Tilde{\nu}_1$. For algorithm $\pi$ and instance $\tilde{\nu}_1$, we denote $
    i={\arg \min} _{a\in \{2,\cdots,K\}}\mathbb{E}_{\pi,\tilde{\nu}_1}[N_a(T)].$ Thus, $\mathbb{E}_{\pi,\tilde{\nu}_1}[N_i(T)]\leq\frac{T}{K-1}$.

Now, consider another instance $\nu_2$ where $r_{a_1},\cdots,r_{a_k}$ are the same as those in $\nu_1$ except the $i$-th arm such that
$$
r_{i}^\prime= \begin{cases}1 / \gamma & \text { with probability of } \frac{7}{10}\gamma^{k} \\ 0 & \text { with probability of } 1-\frac{7}{10}\gamma^{k}\end{cases}
$$
We can verify $\mathbb{E}{[r_{i}^\prime]}=\frac{7}{10}\gamma^{k-1}$ and  $\mathbb{E}{[(r_{i}^\prime)^k]}=\frac{7}{10} \le 1$. Then in $\nu_2$, the arm $i$ is optimal.

Also, we denote the contaminated version of $\nu_2$ as $\Tilde{\nu}_2$. Take $\gamma=\alpha^\frac{1}{k} \in (0,1]$, since for any $a\in[K]$, $\text{TV}(r_a\|r_a^\prime) \le \frac{2}{5}\gamma^k =\frac{2}{5} \alpha \le \frac{\alpha}{1-\alpha}$, from Lemma \ref{lem:HuberTV}, we have for any arm $a\in [K]$, there exists distribution $G_a$ and $G_a^\prime$ such that
$$(1-\alpha) r_a+\alpha G_a=(1-\alpha) r_a^\prime+\alpha G_a^\prime.$$

We consider $\Tilde{\nu}_1=\{x_a=(1-\alpha) r_a+\alpha G_a: a\in [K]\}$ and $\Tilde{\nu}_2=\{x_a^\prime =(1-\alpha) r_a^\prime+\alpha G_a^\prime: a\in [K]\}$.

Now by the classic regret decomposition, we obtain 
$$\mathcal{R}_{T}(\pi,\tilde{\nu}_1) =  (T-\mathbb{E}_{\pi,\tilde{\nu}_1}[N_1(T)])\Delta \ge \mathbb{P}_{\pi,\tilde{\nu}_1}^T \left[N_1(T) \le \frac{T}{2}\right]\frac{T\Delta}{2}. $$
$$\mathcal{R}_{T}(\pi,\tilde{\nu}_2) = \Delta \mathbb{E}_{\pi,\tilde{\nu}_2}[N_1(T)] + \sum_{a\notin\{1,i\}}2\Delta \mathbb{E}_{\pi,\tilde{\nu}_2}[N_a(T)] \ge \mathbb{P}_{\pi,\tilde{\nu}_2}^T \left[N_1(T) \ge \frac{T}{2}\right]\frac{T\Delta}{2}.$$
By applying the Bretagnolle–Huber inequality (\cite{lattimore2020bandit}, Theorem 14.2), we have 

$$\begin{aligned}
 \mathcal{R}_{T}(\pi,\tilde{\nu}_1)+ \mathcal{R}_{T}(\pi,\tilde{\nu}_2)
    & \ge \frac{T\Delta}{2}\left(\mathbb{P}_{\pi,\tilde{\nu}_1}^T \left[N_1(T) \le \frac{T}{2}\right]+\mathbb{P}_{\pi,\tilde{\nu}_2}^T \left[N_1(T) \ge \frac{T}{2}\right]\right). \\
    & \ge \frac{T\Delta}{4} \exp{\left(-\text{KL}\left(\mathbb{P}_{\pi,\tilde{\nu}_1}^T \|\mathbb{P}_{\pi,\tilde{\nu}_2}^T\right)\right)}
\end{aligned}$$

Based on Lemma \ref{lem:DP_KL}, we can get the upper bound of the KL-Divergence between the marginals.
$$
\begin{aligned}
   \text{KL}\left(\mathbb{P}_{\pi,\tilde{\nu}_1}^T \|\mathbb{P}_{\pi,\tilde{\nu}_2}^T\right) & \le 6\epsilon \mathbb{E}_{\pi,\tilde{\nu}_1}\left[\sum_{t=1}^T \text{TV}(x_{a_t}\|x^\prime_{a_t})\right] \\
\end{aligned}
$$

Since, $\text{TV}(x_{a}\|x^\prime_{a})=0$ for $\forall a\in [K]$,  $\Delta=\frac{1}{5}\gamma^{k-1}$ and $\gamma=\alpha^\frac{1}{k}$. We obtain
$$ \mathcal{R}_{T}(\pi,\tilde{\nu}_1) \ge \Omega(T \alpha^{1-\frac{1}{k}}).$$

Combine Gaussian case, case 1 and case 2, we have 
$$\mathcal{R}_T =\Omega\left(\max\left\{\sqrt{KT}, \left(\frac{K}{\epsilon}\right)^{1-\frac{1}{k}}T^{\frac{1}{k}},T \alpha^{1-\frac{1}{k}}\right\}\right).$$

\end{proof}




\section{Proofs of Section \ref{sec:rawupper}}
\begin{proof}[\bf Proof of Theorem \ref{Thm:MeanRaw}]
We denote the finite raw moments distribution for rewards by $P_k$, and denote $P_k$ under $\alpha$-Huber contamination by $P_{\alpha,k}$.
    Let $\hat{\mu}=\frac{1}{n} \sum\limits_{\substack{i \in [n]\\ X_i \sim P_{\alpha,k}}} X_i \mathbbm{1}_{(|X_i| \le M)}$ and $\mu=\mathbb{E}_{X_i \sim P_k}[X_i]$.
    $$\begin{aligned}
        |\Tilde{\mu}-\mu| \le & \left|\text{Lap}\left(\frac{2M}{n\epsilon}\right)\right| +|\hat{\mu}-\mu|\\
         \le  & \left|\text{Lap}\left(\frac{2M}{n\epsilon}\right)\right| +\left|\frac{1}{n} \sum\limits_{\substack{i \in [n]\\ X_i \sim P_{\alpha,k}}} X_i \mathbbm{1}_{(|X_i| \le M)}-\mathbb{E}_{X_i \sim P_{k}} [ X_i \mathbbm{1}_{(|X_i| \le M)}]\right|+|\mathbb{E}_{X_i \sim P_{k}} [ X_i \mathbbm{1}_{(|X_i| \le M)}]-\mu|\\
        = &\left|\text{Lap}\left(\frac{2M}{n\epsilon}\right)\right| +\left|\frac{1}{n} \sum\limits_{\substack{i \in [n]\\ X_i \sim P_{\alpha,k}}} X_i \mathbbm{1}_{(|X_i| \le M)}-\mathbb{E}_{X_i \sim P_{k}} [ X_i \mathbbm{1}_{(|X_i| \le M)}]\right|+|\mathbb{E}_{X_i \sim P_{k}} [ X_i \mathbbm{1}_{(|X_i| > M)}]|\\
        \overset{(a)}{\le} & \frac{2M \log(2/\delta)}{n\epsilon}+\left|\frac{1}{n} \sum\limits_{\substack{i \in [n]\\ X_i \sim P_{\alpha,k}}} X_i \mathbbm{1}_{(|X_i| \le M)}-\mathbb{E}_{X_i \sim P_{k}} [ X_i \mathbbm{1}_{(|X_i| \le M)}]\right|+\mathbb{E}_{X_i \sim P_{k}} [ |X_i |\mathbbm{1}_{(|X_i| > M)}] \quad \text{w.p.} \quad 1-\frac{\delta}{2}\\
        \overset{(b)}{\le} & \frac{2M \log(2/\delta)}{n\epsilon}+\left|\frac{1}{n} \sum\limits_{\substack{i \in [n]\\ X_i \sim P_{\alpha,k}}} X_i \mathbbm{1}_{(|X_i| \le M)}-\mathbb{E}_{X_i \sim P_{k}} [ X_i \mathbbm{1}_{(|X_i| \le M)}]\right|+(\mathbb{E}_{X_i \sim P_{k}} [ |X_i |^k])^{\frac{1}{k}}(\mathbb{P}_{X_i \sim P_{k}}{(|X_i| > M)})^\frac{k-1}{k}\\
         \overset{(c)}{\le}& \frac{2M \log(2/\delta)}{n\epsilon}+\left|\frac{1}{n} \sum\limits_{\substack{i \in [n]\\ X_i \sim P_{\alpha,k}}} X_i \mathbbm{1}_{(|X_i| \le M)}-\mathbb{E}_{X_i \sim P_{k}} [ X_i \mathbbm{1}_{(|X_i| \le M)}]\right|+\frac{1}{M^{k-1}}
    \end{aligned}$$
    where the inequality $(a)$ follows from Lemma \ref{lem:TailLap}, $(b)$ is from H\" older’s Inequality in Lemma \ref{lem:Holder} and $(c)$ follows from Markov's inequality in Lemma \ref{lem:Markov}.

    Now we focus on the upper bound of $\left|\frac{1}{n} \sum\limits_{\substack{i \in [n]\\ X_i \sim P_{\alpha,k}}} X_i \mathbbm{1}_{(|X_i| \le M)}-\mathbb{E}_{X_i \sim P_{k}} [ X_i \mathbbm{1}_{(|X_i| \le M)}]\right|$. Let $N_G$ be the set of indices in $n$ samples distributed according to $G$, and $N_{P_k}$ be the set of indices in $n$ samples distributed according to $P_k$. Then 
    
    \textbf{Case 1: uncontaminated case ($\alpha=0$)} Now, the only thing left is to upper bound $$\left|\frac{1}{n} \sum\limits_{\substack{
    i \in [n]\\
    X_i \sim P_k}} X_i \mathbbm{1}_{(|X_i| \le M)}-\mathbb{E}_{X_i \sim P_{k}} [ X_i \mathbbm{1}_{(|X_i| \le M)}]\right| .$$ For $X_i \sim P_k$, let $Y_i=  X_i \mathbbm{1}_{(|X_i| \le M)}$, then $|Y_i|\le M$ and $\text{Var}(Y_i)=\mathbb{E}[Y_i^2]-(\mathbb{E}[Y_i])^2 \le \mathbb{E}[Y_i^2] \le \mathbb{E}_{X_i \sim P_{k}} [ X_i^2]\le 1$. Then, from Bernstein's inequality in Lemma \ref{lem:Bern}, we have with probability $1-\delta/2$
    \begin{equation}\label{eq:UnconConcen}
        \left|\frac{1}{n} \sum\limits_{\substack{
    i \in [n]\\
    X_i \sim P_k}} X_i \mathbbm{1}_{(|X_i| \le M)}-\mathbb{E}_{X_i \sim P_{k}} [ X_i \mathbbm{1}_{(|X_i| \le M)}]\right|\le  \sqrt{\frac{2\log(4/\delta)}{n}}+\frac{4M \log(4/\delta)}{3n}.
    \end{equation}

    Then we get with probability at least $1-\delta$,$$|\Tilde{\mu}-\mu| \le \sqrt{\frac{2\log(4/\delta)}{n}}+\frac{4M \log(4/\delta)}{3n}+\frac{2M \log(2/\delta)}{n\epsilon}+ \frac{1}{M^{k-1}}.$$

    For $\epsilon >0$ and $\delta \in (0,1)$, we have with probability at least $1-\delta$,$$|\Tilde{\mu}-\mu| \le \sqrt{\frac{2\log(4/\delta)}{n}}+\frac{4M \log(4/\delta)}{n\epsilon}+ \frac{1}{M^{k-1}}.$$

    Taking the truncation threshold $M=\left(\frac{n \epsilon}{4 \log(4/\delta)}\right)^{\frac{1}{k}}$, we have 
    $$|\Tilde{\mu}-\mu| \le \sqrt{\frac{2\log(4/\delta)}{n}}+2\left(\frac{4 \log(4/\delta)}{n\epsilon}\right)^{1-\frac{1}{k}}.$$

     \textbf{Case 2: contaminated case ($\alpha \in (0,\frac{1}{2}]$ )}
    $$\begin{aligned}
       & \left|\frac{1}{n} \sum\limits_{\substack{i \in [n]\\ X_i \sim P_{\alpha,k}}} X_i \mathbbm{1}_{(|X_i| \le M)}-\mathbb{E}_{X_i \sim P_{k}} [ X_i \mathbbm{1}_{(|X_i| \le M)}]\right|\\
       =& \left|\frac{1}{n} \sum\limits_{i \in N_G} X_i \mathbbm{1}_{(|X_i| \le M)}+\frac{1}{n} \sum\limits_{i \in N_{P_k}} X_i \mathbbm{1}_{(|X_i| \le M)}-\mathbb{E}_{X_i \sim P_{k}} [ X_i \mathbbm{1}_{(|X_i| \le M)}]\right|\\
       \le & \underbrace{\left|\frac{1}{n} \sum\limits_{i \in N_G} X_i \mathbbm{1}_{(|X_i| \le M)}\right|}_{T_1}+\underbrace{\left|\frac{1}{n} \sum\limits_{i \in N_{P_k}} X_i \mathbbm{1}_{(|X_i| \le M)}-\mathbb{E}_{X_i \sim P_{k}} [ X_i \mathbbm{1}_{(|X_i| \le M)}]\right|}_{T_2}.
    \end{aligned}$$
     To control $T_1$, we can write it as
$$
\begin{aligned}
    T_1 &=\left|\frac{1}{n} \sum\limits_{i \in N_G} X_i \mathbbm{1}_{(|X_i| \le M)}\right|\\
    &\le\frac{1}{n} \sum\limits_{i \in N_G}  \left|X_i\right| \mathbbm{1}_{(|X_i| \le M)}\\
    & \le \frac{|N_G|}{n} M.
\end{aligned}$$
Then $\frac{|N_G|}{n}$ can be treat as a mean estimation of Bernoulli distribution $Ber(\alpha)$. Then based on Bernstein's inequality in Lemma \ref{lem:Bern}, we get with probability $1-\delta/4$,
$$\left|\frac{|N_G|}{n}-\alpha\right|\le \sqrt{\frac{2\alpha(1-\alpha)\log(8/\delta)}{n}}+\frac{2 \log(8/\delta)}{3n}.$$

Thus,
$$T_1 \le \left(\alpha+\sqrt{\frac{2\alpha\log(8/\delta)}{n}}+\frac{2 \log(8/\delta)}{3n}\right)M.  \quad \text{with probability} 1-\delta/4$$
When $n \ge \frac{\log(8/\delta)}{\alpha}$, we have 
$$T_1 \le 4\alpha M.$$

To bound $T_2$, we have 
$$\begin{aligned}
    &\left|\frac{1}{n} \sum\limits_{i \in N_{P_k}} X_i \mathbbm{1}_{(|X_i| \le M)}-\mathbb{E}_{X_i \sim P_{k}} [ X_i \mathbbm{1}_{(|X_i| \le M)}]\right|\\
    =& \left|\frac{1}{n} \sum\limits_{\substack{
    i \in N_G \cup N_{P_k}\\
    X_i \sim P_k}} X_i \mathbbm{1}_{(|X_i| \le M)}-\frac{1}{n} \sum\limits_{\substack{
    i \in N_G\\
    X_i \sim P_k}} X_i \mathbbm{1}_{(|X_i| \le M)}-\mathbb{E}_{X_i \sim P_{k}} [ X_i \mathbbm{1}_{(|X_i| \le M)}]\right|\\
    \le& \left|\frac{1}{n} \sum\limits_{\substack{
    i \in N_G\\
    X_i \sim P_k}} X_i \mathbbm{1}_{(|X_i| \le M)}\right|+\left|\frac{1}{n} \sum\limits_{\substack{
    i \in [n]\\
    X_i \sim P_k}} X_i \mathbbm{1}_{(|X_i| \le M)}-\mathbb{E}_{X_i \sim P_{k}} [ X_i \mathbbm{1}_{(|X_i| \le M)}]\right|\\
    \le & 4\alpha M +\sqrt{\frac{2\log(16/\delta)}{n}}+\frac{4M \log(16/\delta)}{3n} \quad \text{w.p.} \quad 1-\delta/4\\
\end{aligned}$$
where the last inequality is based on the similar analysis of $T_1$ and the inequality of \eqref{eq:UnconConcen}.

Put everything together, we have with probability at least $1- \delta$, 
 $$|\Tilde{\mu}-\mu| \le  \sqrt{\frac{2\log(16/\delta)}{n}}+\frac{4M \log(16/\delta)}{3n}+8\alpha M+\frac{2M \log(2/\delta)}{n\epsilon}+\frac{1}{M^{k-1}}.$$
 Thus, for $\epsilon >0 $, we have 
 $$|\Tilde{\mu}-\mu| \le  \sqrt{\frac{2\log(16/\delta)}{n}}+\frac{4M \log(16/\delta)}{n\epsilon}+\frac{1}{M^{k-1}}+8\alpha M.$$

Taking $M=\min\left\{\left(\frac{n\epsilon}{4\log(16/\delta)}\right)^{\frac{1}{k}},(8\alpha)^{-\frac{1}{k}}\right\}$, we have
$$|\Tilde{\mu}-\mu| \le  \sqrt{\frac{2\log(16/\delta)}{n}}+2\left(\frac{4 \log(16/\delta)}{n\epsilon}\right)^{1-\frac{1}{k}}+2(8\alpha)^{1-\frac{1}{k}}.$$
 
\end{proof}

%

\begin{proof}[\bf Proof of Theorem \ref{thm:RegRawUp}]
 Let $\tau_0$ be the maximal epoch such that $B_\tau <  \frac{\log (16|\mathcal{S}|\tau^2/\delta)}{\alpha} $. 

 For all epoch $\tau \le \tau_0$, the batch size is less than $2^{\tau_0}$. Since batch size doubles,  until epoch $\tau_0$, we have the number of pulls for each arm $a \in [K]$ is less than $2\cdot 2^{\tau_0} \le 2 \frac{\log (16|\mathcal{S}|\tau_0^2/\delta)}{\alpha}$. Then the regret has to suffer $ \frac{2\log (16|\mathcal{S}|\tau_0^2/\delta)}{\alpha} \Delta_a$ for each $a \in [K]$.

 For $\tau > \tau_0$, $B_\tau \ge  \frac{\log (16|\mathcal{S}|\tau^2/\delta)}{\alpha} $. For each $a \in \mathcal{S}$, from Theorem \ref{Thm:MeanRaw}, we have with probability at least $1-\frac{\delta}{2|\mathcal{S}|\tau^2}$,
$$|\Tilde{\mu}_a-\mu_a| \le  \beta_{\tau}.$$ Given an epoch $\tau > \tau_0$, we denote by $\mathcal{E}_\tau$ the event where for all $a \in \mathcal{S}$ it holds that  $|\Tilde{\mu}_a-\mu_a| \le  \beta_{\tau}.$ and denote $\mathcal{E}=\cup_{\tau > \tau_0}\mathcal{E}_\tau$.By taking union bound, we have
$$\mathbb{P}(\mathcal{E}_\tau)\ge 1-\frac{\delta}{2\tau^2},$$
and $$\mathbb{P}(\mathcal{E})\ge 1-\frac{\delta}{2}\left(\sum_{\tau > \tau_0} \tau^{-2} \right)\ge 1-\delta.$$
In the following, we condition on the good event $\mathcal{E}$. We first show that the optimal arm $a^*$ is never eliminated. For any epoch $\tau >\tau_0$, let $a_\tau=\arg\max_{a\in\mathcal{S}}{\widetilde{\mu}_a}$. Since $$(\widetilde{\mu}_{a_\tau}-\widetilde{\mu}_{a^*})+\Delta_{a_\tau}=|(\widetilde{\mu}_{a_\tau}-\widetilde{\mu}_{a^*})+\Delta_{a_\tau}|\le \left|\widetilde{\mu}_{a_\tau}-\mu_{a_\tau}\right|+\left|\widetilde{\mu}_{a^*}-\mu_{a^*}\right| \le 2\beta_\tau,$$
it is easy to see that the algorithm doesn't eliminate $a^*$.

Then, we show that at the end of epoch $\tau > \tau_0$, all arms such that $\Delta_a \ge 4\beta_\tau$ will be eliminated. To show this, we have that under good event $\mathcal{E}$,
$$\Tilde{\mu}_a+\beta_\tau\le \mu_a+2\beta_\tau <\mu_{a^*}-4\beta_\tau +2\beta_\tau \le \Tilde{\mu}_{a^*}-\beta_\tau \le \Tilde{\mu}_{a_\tau}-\beta_\tau$$
which implies that arm $a$ will be eliminated by the rule. Thus, for each sub-optimal arm $a$, let $\tau(a)$ be the last epoch that arm $a$ is not eliminated. By the above result, we have
$$\Delta_a \le 4\beta_{\tau(a)}=4\sqrt{\frac{2\log(16|\mathcal{S}|\tau(a)^2/\delta)}{B_{\tau(a)}}}+8\left(\frac{4 \log(16|\mathcal{S}|\tau(a)^2/\delta)}{B_{\tau(a)} \epsilon}\right)^{1-\frac{1}{k}}+8(8\alpha )^{1-\frac{1}{k}} .$$

We divide the arms $a\in [K]$ into two groups: $\mathcal{G}_1=\{ a\in [K] : 16(8\alpha )^{1-\frac{1}{k}} \le \Delta_a\}$ and $\mathcal{G}_2=\{ a\in [K] :16(8\alpha )^{1-\frac{1}{k}} \ge \Delta_a\}$.

\textbf{Group 1:} Now, for all arm $a\in \mathcal{G}_1$, we have 
$$\Delta_a \le 8\sqrt{\frac{2\log(16|\mathcal{S}|\tau(a)^2/\delta)}{B_{\tau(a)}}}+16\left(\frac{4 \log(16|\mathcal{S}|\tau(a)^2/\delta)}{B_{\tau(a)} \epsilon}\right)^{1-\frac{1}{k}} .$$

Hence, we have
$$B_{\tau(a)} \le \max\left\{\frac{128 \log(16|\mathcal{S}|\tau(a)^2/\delta)}{\Delta_a^2}, \frac{ 4\log(16|\mathcal{S}|\tau(a)^2/\delta)}{\epsilon}\left(\frac{16}{\Delta_a}\right)^{\frac{k}{k-1}}, \frac{\log (16|\mathcal{S}|\tau_0^2/\delta)}{\alpha}\right\}. $$

Since $|\mathcal{S}|\le K$ and $2^\tau \le T$ for any $\tau$. Thus,

$$B_{\tau(a)} \le \max\left\{\frac{128 \log(16K\log^2 T/\delta)}{\Delta_a^2}, \frac{ 4\log(16K\log^2 T/\delta)}{\epsilon}\left(\frac{16}{\Delta_a}\right)^{\frac{k}{k-1}}, \frac{\log (16K\log^2 T/\delta)}{\alpha}\right\}, $$

Since the batch size doubles, we have $N_a(T) \le 2 B_{\tau(a)}$ for each sub-optimal arm $a$.  Therefore,  for all arm $a\in \mathcal{G}_1$,$$\mathcal{R}_T= \sum_{a\in \mathcal{G}_1} N_a(T)\Delta_a \le 2 B_{\tau(a)}\Delta_a.$$ 

Let $\eta$ be a number in $(0,1)$. For all arms $a \in \mathcal{G}_1$ with $\Delta_a \le \eta$, the regret incurred by pulling these arms is upper bounded by $T\eta$. For any arm $a \in \mathcal{G}_1$ with $\Delta_a > \eta$, choose $\delta=\frac{1}{T}$ and assume $T\ge K$, then the expected regret incurred by pulling arm $a$ is upper bounded by $$\begin{aligned}
    \mathbb{E}\left[\sum_{a\in \mathcal{G}_1, \Delta_a >\eta}\Delta_a N_a(T)\right]
    &\le \mathbb{P}(\Bar{\mathcal{E}})\cdot T +O\left(\sum_{{a\in \mathcal{G}_1, \Delta_a >\eta}} \left\{\frac{\log T}{\Delta_a}+ \frac{\log T}{\epsilon} \left(\frac{1}{\Delta_a}\right)^{\frac{1}{k-1}}+\frac{\log T}{\alpha}\Delta_a\right\}\right)\\
    & \le O\left(\frac{K\log T}{\eta} +\frac{K \log T}{\epsilon \eta^\frac{1}{k-1}}+\frac{K \log T}{\alpha}\right)
\end{aligned}$$
where the last term in the last inequality is based on following result:
from the heavy-tailed assumption for rewards distributions in \eqref{eq:AssumpHeavy}, we have for any $a\in [K]$, $|\mu_a| \le \mathbb{E}_{r_a \sim P_k}|r_a|\le \mathbb{E}_{r_a \sim P_k}|r_a|^k \le 1$, so $\Delta_a =\mu^*-\mu_a \le 2$.

Thus the regret from group 1 is at most $$T \eta +O\left(\frac{K\log T}{\eta} +\frac{K \log T}{\epsilon \eta^\frac{1}{k-1}}+\frac{K \log T}{\alpha}\right).$$

Taking $\eta=\max\left\{\sqrt{\frac{K\log T}{T}},\left(\frac{K\log T}{T\epsilon}\right)^{\frac{k-1}{k}}\right\}$,  the regret from group 1 is at most $$O\left(\sqrt{KT \log T}+\left(\frac{K\log T}{\epsilon}\right)^{\frac{k-1}{k}}T^{\frac{1}{k}}+\frac{K \log T}{\alpha}\right)$$

\textbf{Group 2:} For all other arms $a \in \mathcal{G}_2$, we have the total regret is at most $O(T\Delta_a)=O(T \alpha^{1-\frac{1}{k}})$.

Combine the two groups, choose $\delta=\frac{1}{T}$ and assume $T\ge K$, we have the that the expected regret satisfies,

$$
    \mathcal{R}_T \le O\left(\sqrt{KT \log T}+\left(\frac{K\log T}{\epsilon}\right)^{\frac{k-1}{k}}T^{\frac{1}{k}}+\frac{K \log T}{\alpha}+T \alpha^{1-\frac{1}{k}}\right)
$$

We also give privacy guarantee for the algorithm. Based on Laplacian mechanism in Definition \ref{def:3} and Post-processing in Lemma \ref{PostPro}, we can get that Algorithm \ref{Alg:meta} is $\epsilon$-DP.
\end{proof}

\section{Proofs of Section \ref{sec:CentralUp}}

\begin{proof}[\bf Proof of Theorem \ref{thm:CentralMeanUp}]
    \textbf{Step 1:} we will show that with high probability $1-\delta/2$, $|J-\mu| \le 2r$. 

    To this end, we first study the private histogram. 
    Note $\mathbb{E}[\mathbbm{1}(X_i\in B_j)]=P_{\alpha,k}(B_j)$, then 
    \begin{align*}
        \mathbb{P}\left(|\tilde{p}_j-P_{\alpha,k}(B_j)|>t\right)&=\mathbb{P}\left(\left|\frac{\sum_{i=1}^n \mathbbm{1}(X_i \in B_j)}{n}+\operatorname{Lap}\left(\frac{2}{n\epsilon}\right)-P_{\alpha,k}(B_j)\right|>t\right)\\
        & \le \mathbb{P}\left(\left|\frac{\sum_{i=1}^n \mathbbm{1}(X_i \in B_j)}{n}-P_{\alpha,k}(B_j)\right|>t/2\right)+\mathbb{P}\left(\left|\operatorname{Lap}\left(\frac{2}{n\epsilon}\right)\right|>t/2\right)\\
        & \le 2\exp \left(-\frac{ n t^2}{2}\right)+\exp \left(-\frac{ n\epsilon t}{4}\right),
    \end{align*}
    where the last inequality is from Lemma \ref{lem:Hoef} and Lemma \ref{lem:TailLap}. By a union bound over $j$, we further have
    \begin{align*}
         \mathbb{P}\left(\max_{j \in \mathcal{J}}|\tilde{p}_j-P_{\alpha,k}(B_j)|>t\right) &\le \frac{2D}{r}\left(2\exp \left(-\frac{ n t^2}{2}\right)+\exp \left(-\frac{ n\epsilon t}{4}\right)\right)\\
            &\le 2D\left(2\exp \left(-\frac{ n t^2}{2}\right)+\exp \left(-\frac{ n\epsilon t}{4}\right)\right)
    \end{align*}
Thus, we have with probability $1-\delta/2$, 
\begin{align*}
    \max_{j \in \mathcal{J}}|\tilde{p}_j-P_{\alpha,k}(B_j)| \le \max\left\{\sqrt{\frac{2\ln\frac{16D}{\delta}}{n}},\frac{4\ln\frac{16D}{\delta}}{n\epsilon}\right\} := C_1.
\end{align*}
In the following, we condition on the above event. 
Next, by Chebyshev’s inequality in Lemma \ref{lem:ChebIneq} and the assumption of $\mathcal{P}_k$ that $k$-th central moment is less than 1, we have
\begin{align}
\label{eq:Concen-tail}
        \mathbb{P}_{X \sim \mathcal{P}_{\alpha,k}}\left(|X-\mu|\ge r\right) &\le \alpha \mathbb{P}_{X\sim \mathcal{G}}(|X-\mu|\ge r)+(1-\alpha) \mathbb{P}_{X\sim \mathcal{P}_k}(|X-\mu|\ge r)\nonumber\\
        & \le \alpha+(1-\alpha)(1/r)^k := C_2
\end{align}
Let $j^*$ is the index of the bin containing the true mean $\mu$ and we consider three consecutive intervals $A_{j^*}=B_{j^*-1} \cup B_{j^*} \cup B_{j^*+1} $
   $$ \begin{aligned}
       P_{\alpha,k}(A_{j^*})&=P_{\alpha,k}(B_{j^*-1})+P_{\alpha,k}(B_{j^*})+P_{\alpha,k}(B_{j^*+1})\\
       &\ge P_{\alpha,k}\left((\mu-r,\mu+r)\right)\\
       &\ge 1- C_2.
   \end{aligned}$$
 where the first inequality is from inequality \eqref{eq:Concen-tail}. Now,  for any $j \notin \{j^*-1,j^*,j^*+1\}$, we have when $D\ge 2 r$
\begin{align*}
   \tilde{p}_j &\le P_{\alpha,k}(B_j)+C_1 \le 1- P_{\alpha,k}(A_{j^*})+C_1  \le C_2+C_1.
\end{align*}
On the other hand, since $P_{\alpha,k}(A_{j^*}) \ge 1-C_2$, there must exist some $j \in \{j^*-1,j^*,j^*+1\}$ such that $P_{\alpha,k}(B_{j}) \ge \frac{1-C_2}{3}$. Therefore, for this $j$, we have 
\begin{align*}
     \tilde{p}_j &\ge P_{\alpha,k}(B_{j})- C_1 \ge \frac{1-C_2}{3} - C_1.
\end{align*}
Therefore, if $n$ (depending on $\alpha$, $\epsilon, r$) such that $\frac{1-C_2}{3} - C_1 > C_2 + C_1$, 
the true mean $\mu$ is in the bin chosen by line 3 in Algorithm \ref{Algo:HistCDP2} or it's neighboring bin, which implies that with probability at least $1-\delta/2$, $|J-\mu| \le 2r$.

\textbf{Step 2:} Utilizing the above result, we aim to show that truncation can handle heavy-tail, privacy and robustness in the concentration. 

\begin{align*}
        |\tilde{\mu}-\mu|&=|J + \frac{1}{n}\sum_{i=n+1}^{2n} (X_i-J)\mathbbm{1}(|X_i-J|\le M)+ \operatorname{Lap}\left(\frac{2M}{n\epsilon}\right)-\mu|\\
    & =\left|\frac{1}{n}\sum_{i=n+1}^{2n} (X_i-J)\mathbbm{1}(|X_i-J|\le M)+ \operatorname{Lap}\left(\frac{2M}{n\epsilon}\right)\right.\\
    &\left.+\frac{1}{n}\sum_{i=n+1}^{2n}(J-\mu)\left\{\mathbbm{1}(|X_i-J|\le M)+\mathbbm{1}(|X_i-J|> M)\right\}\right|\\
    & =\left|\frac{1}{n}\sum_{i=n+1}^{2n} (X_i-J+J-\mu)\mathbbm{1}(|X_i-J|\le M)+ \operatorname{Lap}\left(\frac{2M}{n\epsilon}\right)+\frac{1}{n}\sum_{i=n+1}^{2n}(J-\mu)\mathbbm{1}(|X_i-J|> M)\right|\\
    & \le \left|\frac{1}{n}\sum_{i=n+1}^{2n} (X_i-\mu)\mathbbm{1}(|X_i-J|\le M)\right|+ \left|\operatorname{Lap}\left(\frac{2M}{n\epsilon}\right)\right|+\left|\frac{1}{n}\sum_{i=n+1}^{2n}(J-\mu)\mathbbm{1}(|X_i-J|> M)\right|\\
\end{align*}

 We first focus on the first term in the right hand of the last inequality.  Let $N_G$ be the set of indices in $n$ samples distributed according to $G \in \mathcal{G}$, and $N_{P_k}$ be the set of indices in $n$ samples distributed according to $P_k \in \mathcal{P}_k^c$. Then, we have 
 \begin{align*}
         &\left|\frac{1}{n}\sum_{i=n+1}^{2n} (X_i-\mu)\mathbbm{1}(|X_i-J|\le M)\right|\\
    & \le \underbrace{\left|\frac{1}{n}\sum_{i \in N_G} (X_i-\mu)\mathbbm{1}(|X_i-J|\le M)\right|}_{T_1}+\underbrace{\left|\frac{1}{n}\sum_{i \in N_{P_k}} (X_i-\mu)\mathbbm{1}(|X_i-J|\le M)\right|}_{T_2}.
 \end{align*}
 To control $T_1$, we can write it as
$$
\begin{aligned}
    T_1 &=\left|\frac{1}{n}\sum_{i \in N_G} (X_i-\mu)\mathbbm{1}(|X_i-J|\le M)\right|\\
    &\le \frac{1}{n} \sum_{i \in N_G} |(X_i-\mu)|\mathbbm{1}(|X_i-J|\le M)\\
     &\le \frac{1}{n} \sum_{i \in N_G} |(X_i-\mu)|\mathbbm{1}(|X_i-\mu|\le M+2r)\\
    & \le \frac{|N_G|}{n}( M+2r).
\end{aligned}$$

Then $\frac{|N_G|}{n}$ can be treat as a mean estimation of Bernoulli distribution $Ber(\alpha)$. Then based on Bernstein's inequality in Lemma \ref{lem:Bern}, we get with probability $1-\delta/8$,
$$\left|\frac{|N_G|}{n}-\alpha\right|\le \sqrt{\frac{2\alpha(1-\alpha)\log(16/\delta)}{n}}+\frac{2 \log(16/\delta)}{3n}.$$
Thus,
$$T_1 \le \left(\alpha+\sqrt{\frac{2\alpha\log(16/\delta)}{n}}+\frac{2 \log(16/\delta)}{3n}\right)(M+2r),  \quad \text{with probability} 1-\delta/8$$.

Thus, if  $n$ satisfies $\sqrt{\frac{2\alpha\log(16/\delta)}{n}}+\frac{2 \log(16/\delta)}{3n} = O(\alpha)$
, then we have $T_1 = O(\alpha(M+2r))$
Now, we bound $T_2$,
$$
\begin{aligned}
    T_2&=\left|\frac{1}{n}\sum_{\substack{
    i \in N_G \cup N_{P_k}\\
    X_i \sim P_k}
} (X_i-\mu)\mathbbm{1}(|X_i-J|\le M)-\frac{1}{n}\sum_{\substack{
    i \in N_G \\
    X_i \sim P_k}
} (X_i-\mu)\mathbbm{1}(|X_i-J|\le M)\right|\\
&\le \left|\frac{1}{n}\sum_{\substack{
    i \in N_G \cup N_{P_k}\\
    X_i \sim P_k}
} (X_i-\mu)\mathbbm{1}(|X_i-J|\le M)\right|+\left|\frac{1}{n}\sum_{\substack{
    i \in N_G \\
    X_i \sim P_k}
} (X_i-\mu)\mathbbm{1}(|X_i-J|\le M)\right|\\
&\le \left|\frac{1}{n}\sum_{\substack{
    i \in [n]\\
    X_i \sim P_k}
} (X_i-\mu)\mathbbm{1}(|X_i-J|\le M)\right|+ T_1.
\end{aligned}$$

Now we focus on the upper bound of $\left|\frac{1}{n}\sum_{\substack{
    i \in [n]\\
    X_i \sim P_k}
} (X_i-\mu)\mathbbm{1}(|X_i-J|\le M)\right|$.  With probability $1-\delta/8$,
$$\begin{aligned}
    &\left|\frac{1}{n}\sum_{\substack{
    i \in [n]\\
    X_i \sim P_k}
} (X_i-\mu)\mathbbm{1}(|X_i-J|\le M)\right|\\
\le & \left|\frac{1}{n}\sum_{\substack{
    i \in [n]\\
    X_i \sim P_k}
} (X_i-\mu)\mathbbm{1}(|X_i-J|\le M)- \mathbb{E}[(X_1-\mu)\mathbbm{1}(|X_1-J|\le M)]\right|\\
&+|\mathbb{E}[(X_1-\mu)\mathbbm{1}(|X_1-J|\le M)]-\mathbb{E}[(X_1-\mu)]|\\
  \le& \sqrt{\frac{2\log(16/\delta)}{n}}+\frac{4(M+2r) \log(16/\delta)}{3n}+|\mathbb{E}[(X_i-\mu)\mathbbm{1}(|X_i-J|\ge M)]|\\
\le &\sqrt{\frac{2\log(16/\delta)}{n}}+\frac{4(M+2r) \log(16/\delta)}{3n}+\left(\mathbb{E}[|X_i-\mu|^k]\right)^{\frac{1}{k}}\left(\mathbb{P}(|X_i-\mu|\ge M-2r)\right)^{\frac{k-1}{k}}]\\
\le & \sqrt{\frac{2\log(16/\delta)}{n}}+\frac{4(M+2r) \log(16/\delta)}{3n}+\frac{1}{(M-2r)^{k-1}}\\
\le  & \sqrt{\frac{2\log(16/\delta)}{n}}+\frac{4(M+2r) \log(16/\delta)}{3n}+\left(\frac{2}{M}\right)^{k-1}
\end{aligned}$$

where the last inequality follows from $M\ge 4r$ ,
the third inequality follows from H\" older’s Inequality in Lemma \ref{lem:Holder} and the second inequality follows from Bernstein inequality in Lemma \ref{lem:Bern}. That is, let $$Y_i=(X_i-\mu)\mathbbm{1}(|X_i-J|\le M),$$ then $$\begin{aligned}
    |Y_i-\mathbb{E}[Y_i]| &\le  |Y_i|+|\mathbb{E}[Y_i]|\\
    &\le |X_i-\mu|\mathbbm{1}(|X_i-\mu|\le M+2r)+\mathbb{E} [|X_i-\mu|\mathbbm{1}(|X_i-\mu|\le M+2r)]\\
    & \le 2(M+2r)
\end{aligned}$$ and $$\begin{aligned}
    \operatorname{Var}(Y_i-\mathbb{E}[Y_i])&=\mathbb{E}(Y_i-\mathbb{E}[Y_i])^2\le \mathbb{E}[Y_i^2]\\
    &\le  \mathbb{E}_{X_i \sim P_k}[(X_i-\mu)^2\mathbbm{1}(|X_i-\mu|\le M+2r)] \\
    & \le \mathbb{E}_{X_i \sim P_k}[(X_i-\mu)^2]\le 1.
\end{aligned}$$

Therefore, with probability $1-3\delta/8$,
$$T_2\le \sqrt{\frac{2\log(16/\delta)}{n}}+\frac{4(M+2r) \log(16/\delta)}{3n}+\left(\frac{2}{M}\right)^{k-1} + T_1.$$

Now, we focus on the upper bound of $T_3 :=\left|\frac{1}{n}\sum_{i=n+1}^{2n}(J-\mu)\mathbbm{1}(|X_i-J|> M)\right|$.

$$\begin{aligned}
    &\left|\frac{1}{n}\sum_{i=n+1}^{2n}(J-\mu)\mathbbm{1}(|X_i-J|> M)\right|\\
    \le & \frac{1}{n}\sum_{i=n+1}^{2n}\left|J-\mu\right|\mathbbm{1}(|X_i-J|> M)\\
    \le & 2r \frac{\sum_{i=n+1}^{2n}\mathbbm{1}(|X_i-J|> M)}{n}\\
    \le & 2r \frac{\sum_{i=n+1}^{2n}\mathbbm{1}(|X_i-\mu|> M-2r)}{n}
\end{aligned}$$

where $$\begin{aligned}
    \mathbb{E}_{X_i\sim P_{k,\alpha}}[\mathbbm{1}(|X_i-\mu|> M-2r)]&=\mathbb{P}_{X_i\sim P_{k,\alpha}}(|X_i-\mu|> M-2r)\\
    & \le \alpha+(1-\alpha)\mathbb{P}_{X_i\sim P_{k}}(|X_i-\mu|> M-2r)\\
    & \le \alpha+\frac{1}{(M-2r)^k}\le \alpha+\left(\frac{2}{M}\right)^k 
\end{aligned}$$
By Hoeffding's inequality, we have with probability $1-\delta/8$,
$$\frac{\sum_{i=n+1}^{2n}\mathbbm{1}(|X_i-\mu|> M-2r)}{n}\le \mathbb{P}_{X_i\sim P_{k,\alpha}}(|X_i-\mu|> M-2r)+ \sqrt{\frac{\log(16/\delta)}{2n}}.$$
Thus, we have 
$$T_3 = \left|\frac{1}{n}\sum_{i=n+1}^{2n}(J-\mu)\mathbbm{1}(|X_i-J|> M)\right|\le 2r \left(\alpha+\left(\frac{2}{M}\right)^k + \sqrt{\frac{\log(16/\delta)}{2n}}\right ).$$

Putting everything together, we have 
\begin{align*}
    |\tilde{\mu} - \mu| =& O\left(\left(\alpha+\sqrt{\frac{2\alpha\log(16/\delta)}{n}}+\frac{2 \log(16/\delta)}{3n}\right)(M+2r)\right)\\
    +& O\left(\sqrt{\frac{2\log(16/\delta)}{n}}+\frac{4(M+2r) \log(16/\delta)}{3n}+\left(\frac{2}{M}\right)^{k-1} \right)\\
    +& O\left(2r \left(\alpha+\left(\frac{2}{M}\right)^k + \sqrt{\frac{\log(16/\delta)}{2n}}\right ) \right) \\
    +& O\left(\frac{M\log(1/\delta)}{n\epsilon}\right)
\end{align*}

\textbf{Case I: $\alpha = 0$, Uncontaminated concentration.} We want to show that our concentration is better than medians-of-mean in~\cite{kamath2020private} (Theorem 3.5). That is, we are additive for their third term therein (i.e., $
\log(D) + \log(1/\delta)$), while they are multiplicative.

In this case, our $C_2 = (1/r)^k$, and by our first condition on $n$, it need to satisfy $6C_1 + 4C_2 <1$. This implies that $C_2 < 1/4$. Thus, setting $r = 10^{1/k}$ is sufficient. Hence, we have $C_1 < 0.1$, which requires $n$ to satisfy $n \ge 200\log(16D/\delta)$ and $n \ge 20\log(16D/\delta)/\epsilon$. We can safely set $n \ge 200\log(16D/\delta)/\epsilon$.

In the case of $\alpha = 0$, $T_1$ is not a problem, which only introduces another $O(M\log(1/\delta)/n)$. $T_3$ is also not a problem which is dominated by $O((2/M)^{k-1} + \sqrt{\log(1/\delta)}/\sqrt{n})$

Let's summarize all the values: when $\alpha = 0$, $r = 10^{1/k}$ and $n \ge 200\log(16D/\delta)/\epsilon$, we have 
\begin{align*}
    |\tilde{\mu} - \mu| 
 =& O\left(\sqrt{\frac{2\log(16/\delta)}{n}}+\frac{M \log(16/\delta)}{3n}+\left(\frac{2}{M}\right)^{k-1} \right)\\
    +& O\left(\frac{M\log(1/\delta)}{n\epsilon}\right)\\
    =& O\left(\sqrt{\frac{2\log(16/\delta)}{n}}+\frac{M \log(16/\delta)}{\epsilon n}+\left(\frac{2}{M}\right)^{k-1} \right)
\end{align*}

Now, we need to choose $M$ to minimize the above while satisfying $M \ge 4 r$. By standard choice, we set $M = 4 \left(\frac{n\epsilon}{\log(1/\delta)}\right)^{1/k}$, which satisfies $M \ge 4 r$ when $n \ge \frac{10\log(1/\delta)}{\epsilon}$.

\textbf{Case II: $\alpha > 0$. Contaminated concentration.} We want to minimize the term $\mathcal{T}(\alpha,\epsilon)$ while maximizing the possible range of $\alpha$.

In this case, $C_2 = \alpha+(1-\alpha)(1/r)^k$ and again we need to satisfy that $6C_1 + 4C_2 <1$, which first implies that  $\alpha$ needs to be $\alpha < 1/4$. Setting $r = \iota^{1/k}$, we have $C_2 = \alpha + \frac{1}{\iota} (1-\alpha)$, which needs to be less than $1/4$. Let's set $\iota = \frac{1-\alpha}{0.249 - \alpha}$ (hence $\alpha < 0.249$), we have there exists an absolute constant $c_1$ such that when $n \ge c_1\log(16D/\delta)/\epsilon$, we guarantee $6C_1 + 4C_2 <1$.

Now, we turn to $T_1$. If $n \ge \frac{\log(16/\delta)}{\alpha}$ and $M \ge 4r$, we have $T_1 = O(\alpha M)$.

For $T_3$, we have 
\begin{align*}
    T_3 = 2r \left(\alpha+\left(\frac{2}{M}\right)^k + \sqrt{\frac{\log(16/\delta)}{2n}}\right )
\end{align*}

One simple way is to set $n\ge \log(16/\delta)/\alpha^2$. 
Then, we have $T_3 = O(\alpha M + (1/M)^{k-1})$.

Let's summarize it. For any $\alpha \in (0,0.249)$, setting $r = \left(\frac{1-\alpha}{0.249 - \alpha}\right)^{1/k}$. Then, for all $n \ge \max\{c_1\log(16D/\delta)/\epsilon, \log(16/\delta)/\alpha^2\}$, we have
\begin{align*}
   |\tilde{\mu} - \mu|  = O\left(\sqrt{\frac{2\log(16/\delta)}{n}}+\frac{M \log(16/\delta)}{\epsilon n}+\left(\frac{2}{M}\right)^{k-1}  + \alpha M\right)
\end{align*}
Now, we need to choose $M$ to minimize the above while satisfying $M \ge 4 r$. By standard choice, we set $M = \min\{4 \left(\frac{n\epsilon}{\log(1/\delta)}\right)^{1/k}, 4 \alpha^{-1/k}\}$, which satisfies $M \ge 4 r$ when $n$ and $\alpha$ satisfy 
\begin{align*}
    n \ge \frac{\iota\log(1/\delta)}{\epsilon} \quad \text{and} 
    \quad \frac{1}{\alpha} \ge \iota,
\end{align*}
where recall that $\iota = \frac{1-\alpha}{0.249 - \alpha}$. Hence, we only have  a valid concentration for $\alpha \in (0,0.133)$.

\end{proof}

\begin{proof}[\bf Proof of Theorem \ref{thm:central-clean}]
    Let $\tau_0$ be the maximal epoch such that $B_\tau <  \frac{200\log (16 D|\mathcal{S}|\tau^2/\delta)}{\epsilon} $. 

 For all epoch $\tau \le \tau_0$, the batch size is less than $2^{\tau_0}$. Since batch size doubles,  until epoch $\tau_0$, we have the number of pulls for each arm $a \in [K]$ is less than $2\cdot 2^{\tau_0} \le 2 \frac{200\log (16D|\mathcal{S}|\tau_0^2/\delta)}{\epsilon}$. Then the regret has to suffer $ \frac{400\log (16D|\mathcal{S}|\tau_0^2/\delta)}{\epsilon} \Delta_a$ for each $a \in [K]$.

 For $\tau > \tau_0$, $B_\tau \ge  \frac{200\log (16D|\mathcal{S}|\tau^2/\delta)}{\epsilon} $. For each $a \in \mathcal{S}$, from Corollary \ref{cor:clean}, we have with probability at least $1-\frac{\delta}{2|\mathcal{S}|\tau^2}$,
$$|\Tilde{\mu}_a-\mu_a| \le  \beta_{\tau}.$$ Given an epoch $\tau > \tau_0$, we denote by $\mathcal{E}_\tau$ the event where for all $a \in \mathcal{S}$ it holds that  $|\Tilde{\mu}_a-\mu_a| \le  \beta_{\tau}.$ and denote $\mathcal{E}=\cup_{\tau > \tau_0}\mathcal{E}_\tau$.By taking union bound, we have
$$\mathbb{P}(\mathcal{E}_\tau)\ge 1-\frac{\delta}{2\tau^2},$$
and $$\mathbb{P}(\mathcal{E})\ge 1-\frac{\delta}{2}\left(\sum_{\tau > \tau_0} \tau^{-2} \right)\ge 1-\delta.$$
In the following, we condition on the good event $\mathcal{E}$. We first show that the optimal arm $a^*$ is never eliminated. For any epoch $\tau >\tau_0$, let $a_\tau=\arg\max_{a\in\mathcal{S}}{\widetilde{\mu}_a}$. Since $$(\widetilde{\mu}_{a_\tau}-\widetilde{\mu}_{a^*})+\Delta_{a_\tau}=|(\widetilde{\mu}_{a_\tau}-\widetilde{\mu}_{a^*})+\Delta_{a_\tau}|\le \left|\widetilde{\mu}_{a_\tau}-\mu_{a_\tau}\right|+\left|\widetilde{\mu}_{a^*}-\mu_{a^*}\right| \le 2\beta_\tau,$$
it is easy to see that the algorithm doesn't eliminate $a^*$.

Then, we show that at the end of epoch $\tau > \tau_0$, all arms such that $\Delta_a \ge 4\beta_\tau$ will be eliminated. To show this, we have that under good event $\mathcal{E}$,
$$\Tilde{\mu}_a+\beta_\tau\le \mu_a+2\beta_\tau <\mu_{a^*}-4\beta_\tau +2\beta_\tau \le \Tilde{\mu}_{a^*}-\beta_\tau \le \Tilde{\mu}_{a_\tau}-\beta_\tau$$
which implies that arm $a$ will be eliminated by the rule. Thus, for each sub-optimal arm $a$, let $\tau(a)$ be the last epoch that arm $a$ is not eliminated. By the above result, we have
$$\Delta_a \le 4\beta_{\tau(a)}=O\left(\sqrt{\frac{\log(|\mathcal{S}|\tau(a)^2/\delta)}{B_{\tau(a)}}}+\left(\frac{ \log(|\mathcal{S}|\tau(a)^2/\delta)}{B_{\tau(a)} \epsilon}\right)^{1-\frac{1}{k}} \right).$$

Hence, we have
$$B_{\tau(a)} \le O\left(\frac{ \log(|\mathcal{S}|\tau(a)^2/\delta)}{\Delta_a^2}+ \frac{ \log(|\mathcal{S}|\tau(a)^2/\delta)}{\epsilon}\left(\frac{1}{\Delta_a}\right)^{\frac{k}{k-1}}+ \frac{\log (D|\mathcal{S}|\tau_0^2/\delta)}{\epsilon}\right). $$

Since $|\mathcal{S}|\le K$ and $2^\tau \le T$ for any $\tau$. Thus,

$$B_{\tau(a)} \le O\left(\frac{ \log(K\log^2 T/\delta)}{\Delta_a^2}, \frac{ \log(K\log^2 T/\delta)}{\epsilon}\left(\frac{1}{\Delta_a}\right)^{\frac{k}{k-1}}, \frac{\log (D K\log^2 T/\delta)}{\epsilon}\right), $$

Since the batch size doubles, we have $N_a(T) \le 2 B_{\tau(a)}$ for each sub-optimal arm $a$.  Therefore,  for all arm $a\in [K]$,$$\mathcal{R}_T= \sum_{a\in [K]} N_a(T)\Delta_a \le 2 B_{\tau(a)}\Delta_a.$$ 

Let $\eta$ be a number in $(0,1)$. For all arms $a \in [K]$ with $\Delta_a \le \eta$, the regret incurred by pulling these arms is upper bounded by $T\eta$. For any arm $a \in [K]$ with $\Delta_a > \eta$, choose $\delta=\frac{1}{T}$ and assume $T\ge K$, then the expected regret incurred by pulling arm $a$ is upper bounded by $$\begin{aligned}
    \mathbb{E}\left[\sum_{a\in [K], \Delta_a >\eta}\Delta_a N_a(T)\right]
    &\le \mathbb{P}(\Bar{\mathcal{E}})\cdot T +O\left(\sum_{{a\in [K], \Delta_a >\eta}} \left\{\frac{\log T}{\Delta_a}+ \frac{\log T}{\epsilon} \left(\frac{1}{\Delta_a}\right)^{\frac{1}{k-1}}+\frac{\log D T}{\epsilon}\Delta_a\right\}\right)\\
    & \le O\left(\frac{K\log T}{\eta} +\frac{K \log T}{\epsilon \eta^\frac{1}{k-1}}+\frac{KD \log (D T)}{\epsilon}\right)
\end{aligned}$$
where the last term in the last inequality is based on following result:
from the heavy-tailed assumption for rewards distributions in Definition \ref{def:central}, we have for any $a\in [K]$, $\mu_a \in [-D,D]$, so $\Delta_a =\mu^*-\mu_a \le 2D$.

Thus the regret is at most $$T \eta +O\left(\frac{K\log T}{\eta} +\frac{K \log T}{\epsilon \eta^\frac{1}{k-1}}+\frac{KD \log (DT)}{\epsilon}\right).$$

Taking $\eta=\max\left\{\sqrt{\frac{K\log T}{T}},\left(\frac{K\log T}{T\epsilon}\right)^{\frac{k-1}{k}}\right\}$,  the regret is at most $$O\left(\sqrt{KT \log T}+\left(\frac{K\log T}{\epsilon}\right)^{\frac{k-1}{k}}T^{\frac{1}{k}}+\frac{DK \log (DT)}{\epsilon}\right).$$

For privacy guarantee, based on Laplacian mechanism in Definition \ref{def:3}, privacy guarantee for histogram learner in \cite[Lemma 2.3]{karwa2017finite}, parallel composition theorem in Lemma \ref{lem:parallel} and Post-processing in Lemma \ref{PostPro}, we can get the result.

\end{proof}

\begin{proof}[\bf Proof of Theorem \ref{thm:central-cont}]
    The proof of the theorem is similar to the proof of Theorem \ref{thm:RegRawUp}, now the requirement for batch size to start to arm elimination becomes $\max\{\frac{\iota\log(16/\delta)}{\epsilon}, \frac{c_1\log(16D/\delta)}{\epsilon}, \frac{\log(16/\delta)}{\alpha^2}\}$ and the upper bound of $\Delta_a$ for each $a \in [K]$ is $2D$. Then we can get the result of upper bound for regret.

    For privacy guarantee, based on Laplacian mechanism in Definition \ref{def:3}, privacy guarantee for histogram learner in \cite[Lemma 2.3]{karwa2017finite}, parallel composition theorem in Lemma \ref{lem:parallel} and Post-processing in Lemma \ref{PostPro}, we can get the result.
\end{proof}

\end{document}